\documentclass{article}

    \PassOptionsToPackage{numbers, compress}{natbib}

 \usepackage[preprint]{neurips_2026}

\usepackage[utf8]{inputenc} 
\usepackage[T1]{fontenc}    
\usepackage{hyperref}       
\usepackage{url}            
\usepackage{booktabs}       
\usepackage{amsfonts}       
\usepackage{nicefrac}       
\usepackage{microtype}      
\usepackage{xcolor}         

\usepackage{mathtools}
\usepackage{amsmath}
\usepackage{amsthm}
\newtheorem{lemma}{Lemma}
\newtheorem{proposition}{Proposition}
\newtheorem{corollary}{Corollary}
\newtheorem{theorem}{Theorem}
\usepackage{graphicx}
\usepackage{soul}
\usepackage{amssymb} 
\usepackage{hhline}
\usepackage{multirow}
\usepackage[table,xcdraw]{xcolor}
\usepackage{threeparttable}
\usepackage{array}
\usepackage{adjustbox}
\usepackage{caption}
\usepackage{tikz}

\usepackage{colortbl}
\usepackage{amsfonts}
\newcolumntype{L}[1]{$>${\raggedright\let\newline\\\arraybackslash\hspace{0pt}}m{#1}}
\newcolumntype{C}[1]{$>${\centering\let\newline\\\arraybackslash\hspace{0pt}}m{#1}}
\newcolumntype{R}[1]{$>${\raggedleft\let\newline\\\arraybackslash\hspace{0pt}}m{#1}}

\usepackage{algorithm}
\usepackage{algpseudocode}
\usepackage{tabularx} 
\usepackage{subcaption}
\usepackage{wrapfig} 
\usepackage{booktabs} 
\usepackage[most]{tcolorbox}
\definecolor{lightgray}{gray}{0.95} 

\usepackage[most]{tcolorbox}
\usepackage{xcolor}
\usepackage{amsthm}
\usepackage{enumitem}

\newtcolorbox{findingbox}{
  enhanced,
  colback=white,                 
  colframe=green!60!black,       
  boxrule=0.9pt,                 
  arc=3mm,                       
  left=5pt,right=5pt,          
  top=5pt,bottom=5pt,            
  boxsep=0pt,
}

\theoremstyle{definition} 

\theoremstyle{remark}

\theoremstyle{remark}

\title{Don’t Let Bandit Feedback Pull Continual LLM-Recommender Updates Off Target}

\author{%
Taesan Kim$^{1}$ \quad
Hyeongjun Yun$^{1}$ \quad
Jaegul Choo$^{2}$ \quad
Chung Park$^{1}$\thanks{Corresponding Author}
\\
$^{1}$SK Telecom \quad
$^{2}$KAIST
\\
\texttt{\{ktmountain,hjyoon,skt.cpark\}@sk.com} \quad
\texttt{jchoo@kaist.ac.kr}
}

\begin{document}

\maketitle

\begin{abstract}
Generative LLM-based recommenders (LLM-Rec) require continual post-deployment updates, yet deployment logs provide only policy-shaped contextual bandit feedback: outcomes are observed solely for items exposed by a prior serving policy, inducing exposure bias and yielding partial, asymmetric signals consisting of relatively reliable positive responses and ambiguous no-responses. 
We propose an \textbf{Anchored Bandit Policy Optimization (ABPO)} framework for continual LLM-Rec updates that combines group-relative policy optimization (GRPO) with explicit treatment of exposure bias and feedback ambiguity. 
Specifically, we insert the exposed recommendation as a logged anchor into each GRPO rollout group, so that group-relative normalization is calibrated against the action actually exposed by the prior policy rather than against newly sampled rollouts alone.
Because both positive- and no-responses are observed only through prior-policy exposure, we apply self-normalized inverse propensity scoring to the fixed anchor for both feedback types to correct for policy mismatch. 
At the same time, we treat the two feedback types asymmetrically in reliability: positive responses provide relatively direct endorsement signals, whereas no-responses remain ambiguous because they may reflect either true disinterest or unobserved external factors.
To avoid overly aggressive updates from ambiguous no-responses, we temper their penalties with self-certainty, using the model's output-token confidence as a verifier-free reliability signal.
Across five domains from Amazon Reviews and MovieLens, our method yields consistent post-update gains in recommendation accuracy while mitigating prior-policy-induced exposure bias more effectively than prior baselines.
A real-world online A/B test further shows that our method achieves strong CTR and gradual improvement over successive production updates, demonstrating its practical effectiveness in production~\footnote[1]{We released our data and code to support reproducibility.}.

\end{abstract}
\section{Introduction}

Recent LLM-based generative recommender models (LLM-Rec) are rapidly gaining adoption in large-scale recommendation systems, as they can flexibly consume user histories expressed in natural language and generate not only item recommendations but also human-readable rationales~\citep{chen2024softmax, park2026more, park2025towards}. 
However, in real-world deployment, an LLM-Rec operates within a closed feedback loop—\textit{recommendation} $\rightarrow$ \textit{exposure} $\rightarrow$ \textit{user response} $\rightarrow$ \textit{next recommendation}—so its effective data distribution is no longer fixed and instead co-evolves with user behavior over time. 
Practical factors such as user interest drift, seasonality, and changes in UI or exposure mechanisms can therefore quickly invalidate a policy that was optimal at deployment. 
As a result, the LLM-Rec must be continually updated from post-deployment feedback, making continual updating a central challenge for sustaining recommendation quality in production~\citep{yin2025clicks,min2025ctr}.

A fundamental challenge in continual updating is that post-deployment feedback lacks ground-truth labels. 
Instead, it arrives as \textbf{policy-shaped contextual bandit feedback}, where observations are induced by the prior serving policy. Because the logs record outcomes only for exposed items, the resulting data are inherently affected by exposure bias~\citep{han2025reinforcement,min2025ctr}.
Consequently, na\"{i}vely updating the model on such logs can create a self-reinforcing feedback loop, in which the model becomes increasingly biased toward items that were historically favored by the prior policy. 
Moreover, each logged sample provides only a partial and unpaired signal---either an implicit positive response or a silent no-response to the exposed recommendation~\citep{zhang2025leveraging}. 
Existing approaches typically address only one side of this difficulty: some update the model using only positive-response samples, which tends to reinforce previously exposed successes~\citep{han2025reinforcement}, while others treat no-response as an explicit negative signal, which oversimplifies the inherently ambiguous nature of user disengagement~\citep{min2025ctr,zhang2025leveraging}. 
As a result, repeated updates on such biased and ambiguous feedback can progressively deteriorate both recommendation accuracy and diversity relative to the no-update model (Fig.~\ref{fig:intro_1}).

\begin{wrapfigure}[10]{r}{0.6\textwidth}  
\vspace{-1.5em} 
  \centering
  \includegraphics[width=0.6\textwidth]{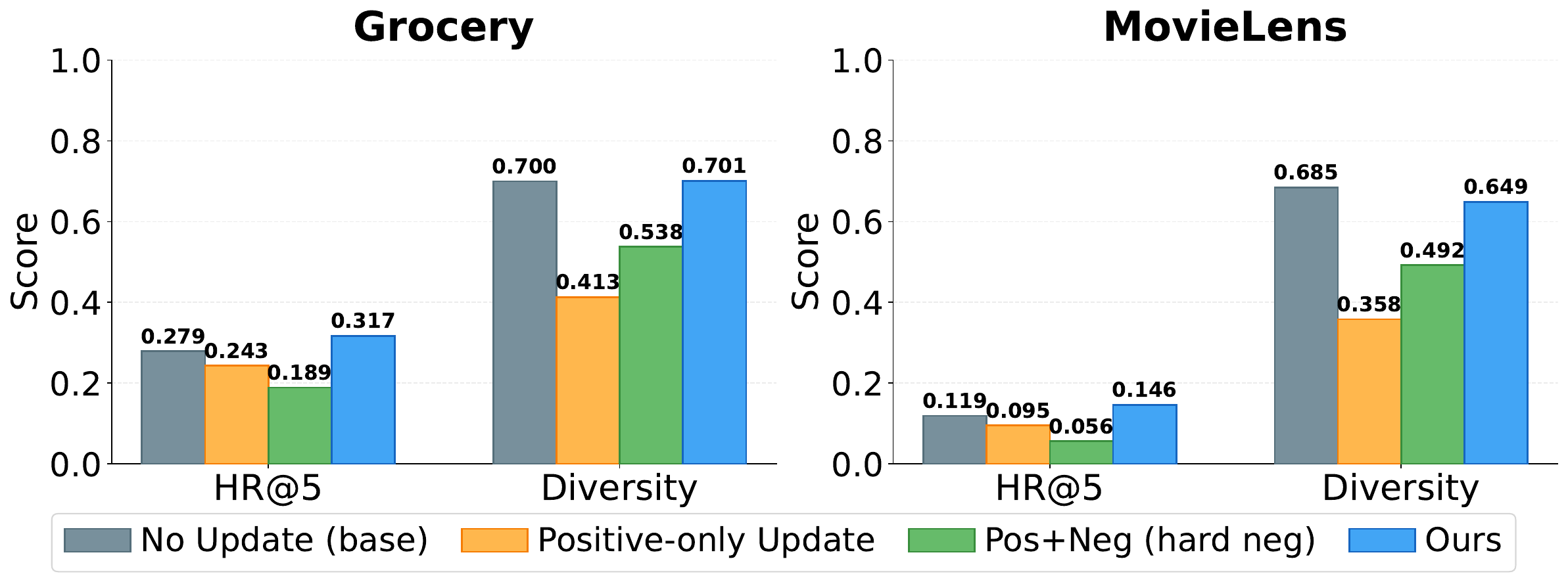}
\caption{
RLVR-based policy updates on contextual-bandit logs degrade both recommendation accuracy and diversity.
}
  \label{fig:intro_1}
\end{wrapfigure}

To address these challenges, we propose an \textbf{Anchored Bandit Policy Optimization (ABPO)} framework for continually updating generative LLM-based recommenders.
Our method builds on the Reinforcement Learning with Verified Rewards (RLVR) paradigm and adopts Group Relative Policy Optimization (GRPO), a representative update rule for RLVR-based LLM-Rec training.
For each logged context, we insert the exposed recommendation as a fixed anchor into the rollout group, so that group-relative normalization is calibrated against the actual action taken by the prior serving policy.
Rather than treating the exposed recommendation as a ground-truth label to imitate, we use it as a \emph{calibration reference} that grounds each update in the policy-shaped feedback actually observed at deployment.
The anchor automatically shifts the group reward baseline according to its observed feedback: positive-response anchors receive high item-matching rewards and raise the baseline, whereas no-response anchors receive low rewards and lower it.
This suppresses blind self-reinforcement on positive exposures and buffers overly aggressive penalties on no-response exposures, stabilizing group-relative updates under biased and partial bandit feedback.

Within this anchored framework, we correct the logged-anchor contribution for exposure bias using self-normalized inverse propensity scoring (SNIPS).
Since both positive- and no-responses are observed only for items exposed by the prior policy, both feedback types inherit off-policy exposure bias; accordingly, we compute SNIPS corrections separately for positive- and no-response logged anchors within each mini-batch.
At the same time, the two feedback types differ in reliability.
A positive response provides a relatively direct endorsement signal, whereas a no-response is inherently ambiguous and may reflect weak interest, exposure effects, or simple inattention rather than explicit rejection.
Thus, treating every no-response as a hard negative can over-penalize otherwise plausible recommendations.
We therefore add a separate \emph{self-certainty} reward for no-response cases, using the model's output confidence as an intrinsic reliability signal when external negative evidence is ambiguous.
This makes deployment feedback bias-corrected and reliability-aware rather than supervised ground truth, enabling stable continual adaptation without overfitting to historically exposed actions or noisy no-response signals.

We evaluate our method on the Amazon Reviews and MovieLens datasets by constructing an initial offline training set together with a simulated post-deployment contextual bandit feedback stream. 
Starting from state-of-the-art LLM-Rec backbones, we update the deployed model using the same feedback logs under (1) our method and (2) widely adopted feedback-driven updating baselines from prior work, and then compare post-update performance on subsequent test data against the pre-update model. 
Across domains, our method achieves strong post-update gains in recommendation accuracy while also improving recommendation diversity (i.e., reducing popularity bias), outperforming existing approaches in most settings.
Beyond offline evaluation, ABPO has been deployed in our real-world service, where it generally achieves higher click-through-rate than competing update strategies while showing gradual improvement over successive production updates.
It remains actively used in production, demonstrating its practical effectiveness under partial observability and off-policy logging.

\vspace{-0.3cm}
\section{Background and Motivation}
\vspace{-0.2cm}
\subsection{Contextual Bandit Feedback}
\vspace{-0.2cm}
In real-world deployment, feedback logs collected over time are inherently off-policy because they are not passive records of user preference, but interaction data generated under the intervention of the recommender policy that was live at the time.
Under this mismatch, na\"{\i}vely treating logged bandit feedback as the ground-truth labels causes the model to learn not the true user preference but the exposure distribution induced by the prior serving policy.
Because feedback is observable only for items the historical policy chose to expose, reinforcing positive-response logs encourages the model to repeatedly regenerate items that were frequently shown and happened to succeed, yielding strong \emph{self-reinforcement} and amplifying popularity bias while suppressing long-tail items.
Conversely, interpreting a no-response as an explicit rejection is even more problematic: no-responses can arise from limited attention, UI placement and presentation effects, time constraints, or simple absence of reaction, so treating them as hard negatives introduces abundant false negatives and drives overly aggressive, biased updates that collapse diversity.

\begin{wrapfigure}[11]{r}{0.44\textwidth}  
\vspace{-2.1em} 
  \centering
  \includegraphics[width=0.44\textwidth]{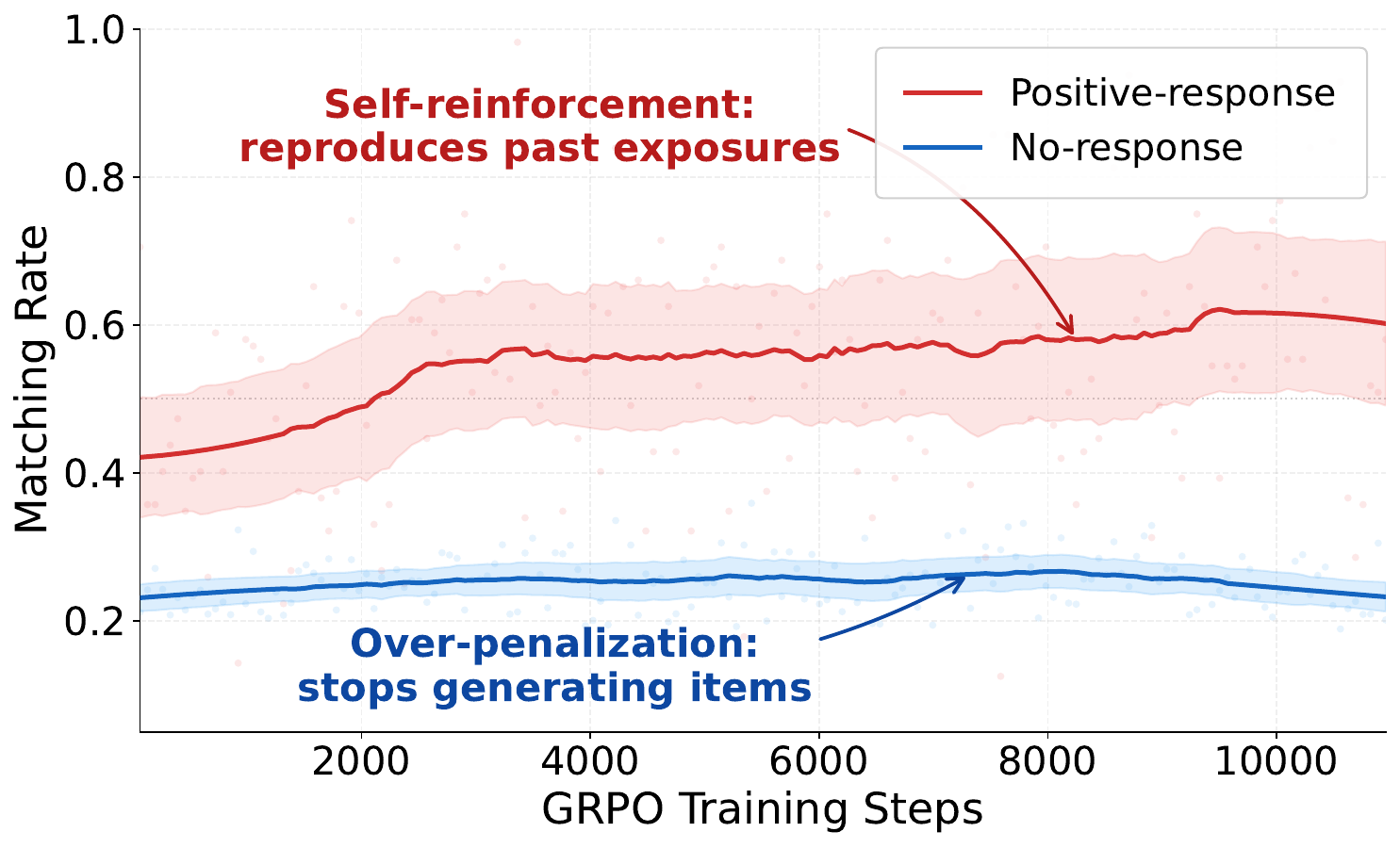}
\caption{
Item-matching reward values for exposed items during GRPO update training.
}
  \label{fig:intro}
\end{wrapfigure}

Figure~\ref{fig:intro} substantiates these failure modes.
When we repeatedly apply Group Relative Policy Optimization (GRPO) across successive post-deployment update rounds, where the model is periodically updated with newly accumulated contextual-bandit feedback, under a reward design that assigns +1 when the model matches a positive-response exposure and -1 when it matches a no-response exposure, the matching rate for positive-response exposures steadily increases, approaching $\sim$60\% in later updates.
This indicates that the policy progressively reproduces items that the prior policy exposed and that happened to receive positive feedback.
In contrast, the matching rate for no-response exposures quickly decays toward a small value, as repeated penalization makes the policy stop generating those items altogether.
Together, these trends indicate that contextual-bandit updating can over-amplify past exposed successes while over-penalizing ambiguous disengagement, ultimately degrading generalization and exacerbating popularity bias.
These observations motivate the need to correct the bias of contextual-bandit logs by explicitly accounting for the asymmetric reliability of positive- versus no-response feedback.

\vspace{-0.3cm}
\subsection{LLM-Rec Update Setting}
\vspace{-0.2cm}

To formalize (1) how we first train the deployed recommender via supervised learning and then (2) continually update it using post-deployment feedback, we begin by defining the recommendation task and the deployment-feedback setting.

\textbf{Definition 1. User Context and Supervised Next-Item Prediction.}
For a user, let $\mathbf{x}_t = (x_1, x_2, \dots, x_t)$ denote the interaction history up to time $t$, where each $x_i$ represents a previously interacted item together with its associated metadata and chronological context.
In the supervised setting, the ground-truth next item is denoted by $x_{t+1} \in \mathcal{V}$, where $\mathcal{V}$ is the item set.
The supervised training dataset is defined as \setlength{\fboxsep}{0pt}\colorbox[HTML]{DAE8FC}{$\mathcal{D}_{\mathrm{sup}} = \{(\mathbf{x}_t, x_{t+1})\}$.}
We obtain an initial deployed policy $\pi_{\theta_{\mathrm{init}}}$ by training a generative recommender on $\mathcal{D}_{\mathrm{sup}}$ using a standard LLM training recipe, consistent with prior LLM-Rec training paradigms used in works such as Rec-R1~\citep{lin2025rec} and TWiCE-Rec~\citep{park2026think}, resulting in the objective $\mathcal{L}_{\mathrm{sup}}=\mathbb{E}_{(\mathbf{x}_t, x_{t+1}) \in \mathcal{D}_{\mathrm{sup}}}\left[\log \pi_{\theta_{\mathrm{init}}}(x_{t+1}\mid \mathbf{x}_t)\right].$

\textbf{Definition 2. Post-Deployment Contextual Bandit Feedback.}
After deployment, \(\pi_{\theta_{\mathrm{init}}}\) no longer receives full
next-item supervision for every candidate action. For each later context
\(\mathbf{x}_{t'}=(x_1,\dots,x_{t'})\) with \(t'>t\), a candidate set
\(\mathcal C_{t'}\subseteq\mathcal V\) is specified by the deployment or
evaluation protocol, and the logging policy samples
\(a^{\log}\in\mathcal C_{t'}\) from \(e_0(\cdot\mid \mathbf{x}_{t'})\). Here,
\(e_0(a\mid \mathbf{x}_{t'})\) abbreviates the candidate-conditional
probability \(e_0(a\mid \mathbf{x}_{t'},\mathcal C_{t'})\); we omit
\(\mathcal C_{t'}\) for readability. In our offline experiments,
\(\mathcal C_{t'}\) is instantiated as described in
Appendix~\ref{app:dataset_overall}.
The environment then returns partial feedback \(y\in\{1,0\}\) only for the exposed item, where \(y=1\) denotes an observed positive response, such as a click, and \(y=0\) denotes a no-response.
Importantly, \(y=0\) is not treated as a deterministic negative label or as a mismatch with a fixed ground-truth next item; it may also result from non-examination, timing, position bias, or other unobserved factors.
The resulting post-deployment bandit dataset is defined as
\setlength{\fboxsep}{0pt}\colorbox[HTML]{DAE8FC}{$
\mathcal{D}_{\mathrm{bandit}}
=
\{(\mathbf{x}_{t'}, a^{\log}, y, e_0(a^{\log}\mid \mathbf{x}_{t'}))\}
$},
where \(e_0(a^{\log}\mid \mathbf{x}_{t'})\) is the item-level logging propensity, i.e., the probability that the logging policy exposes item \(a^{\log}\) under context \(\mathbf{x}_{t'}\).
In our offline bandit construction, this propensity is known because the logged recommendation is sampled from a controlled stochastic item-level recommendation distribution.
Concretely, for each candidate item \(a\in\mathcal C_{t'}\), we compute \(s_0(a,\mathbf{x}_{t'})\), the length-normalized average token log-probability assigned by the LLM to the textual representation of item \(a\), and define
$
e_0(a\mid \mathbf{x}_{t'})
=
\frac{\exp(s_0(a,\mathbf{x}_{t'})/\tau)}
{\sum_{b\in\mathcal C_{t'}}\exp(s_0(b,\mathbf{x}_{t'})/\tau)}.
$
The logged recommendation is then sampled as
$
a^{\log}\sim \mathrm{Categorical}(e_0(\cdot\mid \mathbf{x}_{t'})).
$
Thus, the stored propensity is the normalized item-level exposure probability induced by the softmax distribution over candidate items, not the raw token-level likelihood of the generated recommendation text.
Unlike \(\mathcal{D}_{\mathrm{sup}}\), \(\mathcal{D}_{\mathrm{bandit}}\) is policy-shaped: feedback is observed only for the recommendation exposed by the logging policy, while counterfactual outcomes for unexposed items remain unobserved.
Therefore, the dataset is inherently partial and biased by the exposure pattern induced by \(e_0\).

\textbf{Definition 3. Continual Policy Update Problem Setting.}
Given an initial policy $\pi_{\theta_{\mathrm{init}}}$ trained on $\mathcal{D}_{\mathrm{sup}}$, we collect a post-deployment contextual-bandit dataset
$
\mathcal{D}_{\mathrm{bandit}}
$.
Our goal is to train an updated policy $\pi_{\theta}$ that improves future post-deployment recommendations.
Ideally, this corresponds to maximizing
$
J(\pi_{\theta})
=
\mathbb{E}_{\mathbf{x}\sim p_{\mathrm{deploy}}}
\mathbb{E}_{a\sim\pi_{\theta}(\cdot\mid\mathbf{x})}
[r(\mathbf{x},a)],
$
where $p_{\mathrm{deploy}}$ is the context distribution underlying $\mathcal{D}_{\mathrm{bandit}}$, and $r(\mathbf{x},a)$ is the latent user reward.
However, $\mathcal{D}_{\mathrm{bandit}}$ provides feedback only for the logged action $a^{\log}$ exposed by $\pi_{\theta_{\mathrm{init}}}$, making the update problem one of learning from partial, off-policy contextual-bandit evidence rather than supervised ground truth.
However, this optimization must rely solely on $\mathcal{D}_{\mathrm{bandit}}$, which provides only partial supervision---feedback is observed only for recommendations previously exposed by $\pi_{\theta_{\mathrm{init}}}$.
Thus, our problem is to continually update a deployed policy from policy-shaped contextual bandit feedback, rather than to solve a standard supervised next-item prediction problem.

\vspace{-0.3cm}
    \begin{figure*}[]
    \begin{center}
    \includegraphics[width=0.9999\linewidth]{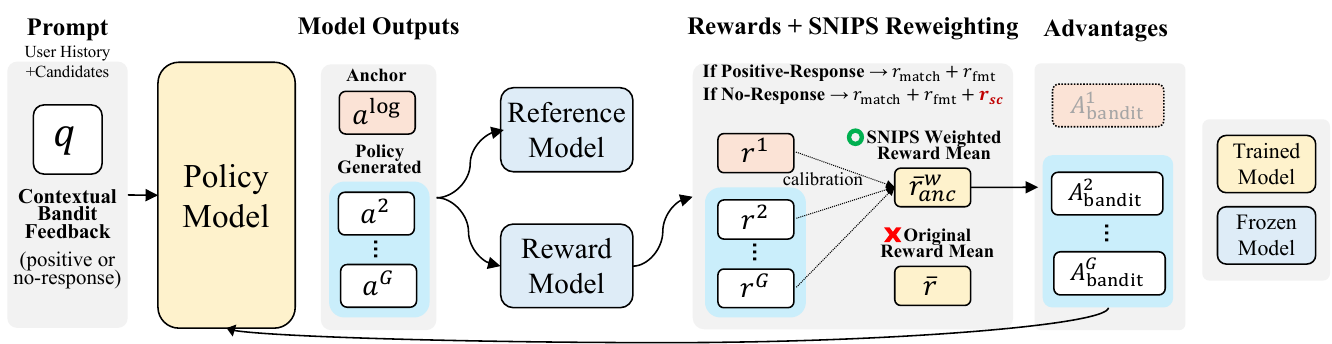}
    \end{center}
\caption{Overview of Anchored Contextual Bandit Reinforcement Learning}
    \label{fig:feedback_framework}
    \end{figure*}

\vspace{-0.1cm}
\section{Anchored Bandit Policy Optimization (ABPO)}
\vspace{-0.3cm}

\subsection{From Vanilla GRPO to Contextual Bandit Feedback}
\vspace{-0.2cm}
RLVR methods such as Group Relative Policy Optimization (GRPO)~\citep{shao2024deepseekmath, guo2025deepseek} serve as an effective framework for continual model updates, allowing the generative policy to be iteratively refined by contrasting diverse recommendation candidates against task-specific reward signals~\citep{yin2025clicks}.
Formally, in the LLM-Rec update task, we construct a prompt $q$ consisting of the user context $\mathbf{x}_t$ and a candidate recommendation item set, and ask the model to generate the most suitable item for the target user.
Then, given a prompt \(q\), the policy \(\pi_{\theta_{\text{old}}}\) samples \(G\) outputs \(\{a^{j}\}_{j=1}^{G}\).
Here, $\pi_{\theta_{\text{old}}}$ is the policy from the previous GRPO step, whereas $\pi_{\theta_{\text{init}}}$ is the initially deployed model---these two are distinct.
Assuming \(K\) reward functions, the aggregated reward for the \(j\)-th rollout is computed as
\begin{equation}
r_{\text{sum}}^{j} = r_{1}^{j} + r_{2}^{j} + \cdots + r_{K}^{j},
\label{eq:grpo_reward_sum_app}
\end{equation}
where \(r_k^{j}\) denotes the \(k\)-th reward evaluated on rollout \(a^{j}\).
In recommendation, a commonly used reward is the \setlength{\fboxsep}{0pt}\colorbox[HTML]{DAE8FC}{item-matching reward}, which assigns $+1$ if a rollout $a^{j}$ matches a positive-response exposure, $-1$ if it matches a no-response exposure, and $0$ otherwise.
The group-relative advantage for the \(j\)-th rollout is then computed by normalizing the aggregated rewards within the rollout group:
\begin{equation}
A^{j}_{\text{sum}} =
\frac{
r_{\text{sum}}^{j} - \mathrm{mean}\{ r_{\text{sum}}^{1}, \ldots, r_{\text{sum}}^{G}\}
}{
\mathrm{std}\{ r_{\text{sum}}^{1}, \ldots, r_{\text{sum}}^{G} \}
}.
\label{eq:grpo_advantage_app}
\end{equation}
The corresponding multi-reward GRPO optimization objective is
\begin{equation}
\scalebox{0.92}{$
\mathcal{J}_{\mathrm{GRPO}}(\theta) = 
\mathbb{E}_{\{a^{j}\}_{j=1}^G \sim \pi_{\theta_{\mathrm{old}}}(\cdot|q)} 
\left[
\frac{1}{G} \sum_{j=1}^{G}
\min \left(
s^{j}(\theta)\, A^{j}_{\text{sum}},\;
\mathrm{clip}(s^{j}(\theta), 1 - \epsilon, 1 + \epsilon)\,A^{j}_{\text{sum}}
\right)
\right],
$}
\label{eq:grpo_app}
\end{equation}
where $s^{j}(\theta) = \frac{\pi_{\theta}(a^{j} \mid q)}{\pi_{\theta_{\mathrm{old}}}(a^{j} \mid q)}$ and \(\epsilon\) denotes the clipping threshold. 
For simplicity, we leave out the KL-divergence regularizer from this formulation.
Since $a^{\log}$ is simply the item exposed by the prior deployed policy under partial observability, treating it as a pseudo ground-truth target can create a self-reinforcing feedback loop.
Therefore, an effective update method should explicitly account for contextual-bandit bias and calibrate how strongly the exposed recommendation influences policy optimization.
Fig.~\ref{fig:feedback_framework} illustrates the architecture of our framework.

\vspace{-0.3cm}
\subsection{Anchor-Augmented Rollout Group}
\label{subsec:anchor_group}
\vspace{-0.3cm}

To account for the fact that all logged feedback is filtered through the exposure mechanism of the prior serving policy, we modify GRPO's rollout-group construction by injecting the exposed recommendation into each rollout group as a fixed anchor.
For each logged prompt $q$, we generate a rollout group of size $G$ under the previous-step policy $\pi_{\theta_{\text{old}}}$, denoted by $\{a^{1}, a^{2}, \ldots, a^{G}\}$, and explicitly replace the first rollout with the exposed recommendation, i.e., \setlength{\fboxsep}{0pt}\colorbox[HTML]{DAE8FC}{$a^{1}=a^{\log}$}.
The remaining \(G-1\) rollouts are sampled from the current rollout policy: $a^{j} \sim \pi_{\theta_{\text{old}}}(\cdot \mid q), \;\; j=2,\ldots,G$.
This anchor-augmented construction is the key departure from vanilla GRPO. 
Rather than evaluating newly sampled rollouts in isolation, this design places them in direct comparison with the historically exposed recommendation under the same logged context.
As a result, the exposed recommendation becomes part of the reward statistics used for group-relative normalization, allowing the mean reward of the group to reflect the behavioral baseline induced by the previous policy.

This design is particularly important under policy-shaped contextual bandit feedback.
Without anchoring, a sampled rollout rarely matches the logged exposure in a large item space, making the item-matching reward extremely sparse.
Anchoring ensures that each group contains at least one rollout tied to the actually observed feedback, so the sparse bandit signal reliably enters group-relative normalization.
When the exposed recommendation receives a positive response, anchoring $a^{\log}$ makes it a high-reward reference, raising the group mean and reducing the relative advantage of rollouts that simply copy the logged action.
This discourages blind self-reinforcement of the incumbent policy's exposure tendency.
Conversely, when the exposed recommendation receives no response, anchoring $a^{\log}$ makes it a low-reward reference, lowering the group mean and preventing penalties on alternative rollouts from being overly amplified.
Thus, the anchor serves as a context-dependent calibration point for reward normalization under sparse contextual-bandit feedback.

\vspace{-0.2cm}
\subsection{Self-Normalized Inverse Propensity Scoring (SNIPS) Reweighting}
\vspace{-0.2cm}

For the $i$-th example in a mini-batch, the logged anchor $a_i^{\log}$ is an exposed recommendation produced by the initially deployed logging policy $\pi_{\theta_{\mathrm{init}}}$, whereas the non-anchor rollouts $a_i^{2{:}G}$ are sampled from the previous-step GRPO policy $\pi_{\theta_{\mathrm{old}}}$. 
Therefore, the logged anchor reflects the historical exposure distribution and is used only to calibrate the group-relative baseline, not as an additional on-policy rollout in the GRPO surrogate.
To account for this exposure mismatch, we use a Self-Normalized Inverse
Propensity Scoring (SNIPS)-style weight only when the logged anchor contributes
to the group baseline. For each logged anchor, we define the item-level
propensity ratio
\begin{equation}
w_i^{\log}
=
\mathbf{sg}
\left[
\frac{
e_{\mathrm{old}}(a_i^{\log}\mid q_i)
}{
e_0(a_i^{\log}\mid q_i)
}
\right],
\label{eq:ips_weight}
\end{equation}
where \(e_0(a_i^{\log}\mid q_i)\) denotes the item-level logging propensity
under the initial logging policy, and \(e_{\mathrm{old}}(a_i^{\log}\mid q_i)\)
denotes the item-level exposure probability assigned to the same logged item
by the rollout policy used to generate the current group.
Both propensities are computed over the same candidate item set \(\mathcal C_i\).
Specifically, for \(p\in\{0,\mathrm{old}\}\), we compute a length-normalized
LLM token score \(s_p(a,q_i)\) for each \(a\in\mathcal C_i\), and define
$
e_p(a\mid q_i)
=
\frac{\exp(s_p(a,q_i)/\tau)}
{\sum_{b\in\mathcal C_i}\exp(s_p(b,q_i)/\tau)}.
$
Thus, the ratio in Eq.~\eqref{eq:ips_weight} is an item-level propensity ratio,
not a ratio of raw token-level likelihoods. 
The stop-gradient operator
\(\mathbf{sg}[\cdot]\) prevents the propensity ratio from directly contributing
gradients to the policy update.
Since positive- and no-responses differ in reliability and scale,
we self-normalize the weights separately for the two feedback types.
Let $\mathcal{B}_{1}$ and $\mathcal{B}_{0}$ denote the sets of positive-response
and no-response anchors in the mini-batch, respectively. For each feedback
type $y\in\{1,0\}$, define
\begin{equation}
\bar{w}^{\log}_{y}
=
\frac{1}{|\mathcal{B}_{y}|}
\sum_{i\in\mathcal{B}_{y}} w_i^{\log},
\qquad
\hat{w}_{i}^{\log}
=
\frac{w_i^{\log}}{\bar{w}^{\log}_{y}+\delta}.
\label{eq:snips_weight_split}
\end{equation}

The resulting weight is used to form a SNIPS weighted reward mean:
\begin{equation}
\bar r^{\,w}_{i,\mathrm{anc}}
=
\frac{
\hat{w}_{i}^{\log} r_i^{\log}
+
\sum_{j=2}^{G} r_i^j
}{
\hat{w}_{i}^{\log}+G-1
}.
\label{eq:weighted_anchor_baseline}
\end{equation}
The advantages used in the GRPO surrogate are then computed only for the policy-sampled rollouts:
$
A_{i,\mathrm{bandit}}^{j}
=
(r_i^j-\bar r^{\,w}_{i,\mathrm{anc}}) / \sigma_i
,
\;\; j=2,\ldots,G,
$
where $\sigma_i$ is the weighted standard deviation of the anchor-augmented reward group, computed with the same anchor weight used in $\bar r^{\,w}_{i,\mathrm{anc}}$.
Specifically,
$
\sigma_i
=
\sqrt{
\frac{
\hat w_i^{\log} \left(r_i^{\log}-\bar r^{\,w}_{i,\mathrm{anc}}\right)^2
+
\sum_{j=2}^{G}
\left(r_i^j-\bar r^{\,w}_{i,\mathrm{anc}}\right)^2
}{
\hat w_i^{\log}+G-1
}
+\epsilon_{\mathrm{std}}
}.
$
Thus, SNIPS is applied to both positive- and no-response anchors because
both are historical exposures generated by the logging policy. However, the
logged anchor affects learning only through the weighted group baseline; it
does not receive a separate PPO-ratio term or direct off-policy surrogate
update. The sampled non-anchor rollouts are optimized with the standard
GRPO likelihood ratio
$
s_i^j(\theta)
=
\frac{
\pi_{\theta}(a_i^j\mid q_i)
}{
\pi_{\theta_{\mathrm{old}}}(a_i^j\mid q_i)
},
\;\; j=2,\ldots,G.
$

\vspace{-0.4cm}
\subsection{Asymmetric Bandit-Aware Reward Construction}
\vspace{-0.2cm}
In recommendation, the two feedback cases in contextual bandit logs pose fundamentally different challenges: positive-response events primarily suffer from off-policy bias, whereas no-response events cannot be safely interpreted as clean negatives.
Motivated by this asymmetry, we design an asymmetric update mechanism, where positive- and no-response events affect policy optimization through distinct GRPO objective instantiations.

\vspace{-0.3cm}
\subsubsection{Reward Function}
\vspace{-0.2cm}
\paragraph{Positive-Response Case.}
For positive-response events, the exposed recommendation $a^{\text{log}}$ corresponds to an item to which the user responded favorably after being exposed by the historical policy.  
First, we define the item-matching reward as
$
r_{\text{match}} = \mathbb{I}\!\left[a^{j} = a^{\text{log}}\right],
$
which provides a binary signal indicating whether the rollout coincides with the logged action. 
This captures a core objective in recommendation, namely, improving recommendation accuracy.
Second, we define the format compliance reward $r_{\text{fmt}}$ as a binary signal for whether the output follows the prescribed schema: the item ID must appear within \texttt{<item\_id> ... </item\_id>} and the item title within \texttt{<item> ... </item>}.
In industrial generative recommenders, such format compliance is not an auxiliary preference but a minimum requirement for reliable deployment.
We aggregate these reward components as in Eq.~\eqref{eq:grpo_reward_sum_app}, and then normalize the aggregated rewards (i.e., $r_{\text{match}}+r_{\text{fmt}}$) over all anchor-augmented rollouts according to Eq.~\eqref{eq:grpo_advantage_app} to compute the advantage \(A_{\text{sum}}^{j}\).

\vspace{-0.3cm}
\paragraph{No-Response Case.}
In this case, a rollout should be penalized when it matches $a^{\text{log}}$, as it reproduces an exposure that did not receive user engagement.
Accordingly, we define the item-matching reward as
$
r_{\text{match}}^{j} = -\mathbb{I}\!\left[a^{j} = a^{\text{log}}\right],
$
which assigns a penalty when the rollout coincides with the exposed recommendation.
This reward discourages the policy from repeatedly reproducing recommendations associated with no-response outcomes.
However, a no-response is not a clean rejection.
It may reflect weak interest, limited exposure, timing effects, or simple inattention, rather than explicit negative preference.
Thus, relying only on the item-matching penalty can over-penalize plausible recommendations under ambiguous feedback.
To mitigate this issue, we incorporate self-certainty~\citep{zhao2025learning} as an auxiliary intrinsic reward for no-response events.
Self-certainty provides a verifier-free confidence signal from the model itself, allowing the update to consider whether a generated recommendation is internally coherent and confidently supported by the policy when external negative evidence is ambiguous.
Specifically, for a rollout $a^{j}$, its self-certainty is defined as
$
r_{\mathrm{sc}}^{j}
=
\frac{1}{T}
\sum_{t=1}^{T}
\mathrm{KL}\!\left(
U \,\|\, \pi_{\theta}(\cdot \mid q, a^{j}_{<t})
\right),
$
where \(U\) denotes the uniform distribution over the token vocabulary, and \(a^{j}_{<t} = (a^{j}_1,\dots,a^{j}_{t-1})\) denotes the token prefix before the \(t\)-th token.
A larger value indicates a more concentrated output distribution and hence higher model confidence.
Therefore, self-certainty does not replace the no-response penalty; rather, it tempers updates by complementing the sparse and ambiguous item-matching signal with an intrinsic reliability signal.
Compared with entropy- or perplexity-based confidence measures, self-certainty is less sensitive to length bias and avoids directly encouraging degenerate low-entropy outputs, making it a more suitable sequence-level confidence signal for this setting~\citep{zhao2025learning}.
As in the positive-response case, the format compliance reward is also applied to no-response events.
We aggregate the resulting reward components as
$
r^{j}
=
r_{\text{match}}^{j}
+
r_{\text{fmt}}^{j}
+
\lambda_{\mathrm{sc}} r_{\mathrm{sc}}^{j},
$
and normalize them across all anchor-augmented rollouts to compute the advantage \(A_{\text{sum}}^{j}\).

\vspace{-0.3cm}

\subsection{Anchor-Constrained Optimization Objective}
\label{subsec:final_objective}
\vspace{-0.2cm}
With the SNIPS-weighted anchored baseline defined above, the final training
objective keeps the same GRPO surrogate form, but applies it only to
the policy-sampled non-anchor rollouts. 
The final objective is
\begin{equation}
\scalebox{0.85}{$
\mathcal{J}_{\mathrm{ABPO}}(\theta)
=
\mathbb{E}_{i,\; a_i^{2{:}G}\sim \pi_{\theta_{\mathrm{old}}}(\cdot\mid q_i)}
\left[
\frac{1}{G-1}
\sum_{j=2}^{G}
\min\left(
s_i^j(\theta)\,A_{i,\mathrm{bandit}}^j,
\;
\mathrm{clip}(s_i^j(\theta),1-\epsilon,1+\epsilon)\,
A_{i,\mathrm{bandit}}^j
\right)
\right],
$}
\label{eq:final_bagrpo}
\end{equation}
where
$
s_i^j(\theta)
=
\frac{
\pi_{\theta}(a_i^j\mid q_i)
}{
\pi_{\theta_{\mathrm{old}}}(a_i^j\mid q_i)
}.
$
Therefore, our method preserves the optimization backbone of GRPO while
modifying two core components to fit logged recommendation feedback:
(1) the group-relative baseline is anchor-calibrated using the historically
exposed action, without treating it as an additional on-policy rollout in the surrogate; and
(2) the scalar utility used for group-relative comparison is redesigned to
treat positive- and no-response logs differently.
This yields a simple and practical extension of GRPO for continual updating
from bandit feedback in generative recommendation.
Note that mini-batches are stratified to include both positive- and
no-response examples, which stabilizes SNIPS normalization and optimization
under severe feedback imbalance.

\vspace{-0.3cm}
\paragraph{Remark: Baseline Calibration by Anchored Rollouts.}
Fix a logged context $q$ and its exposed action $a^{\log}$.
In our anchored GRPO update, the logged action is not treated as an
additional on-policy rollout in the surrogate. Instead, it is used
only to calibrate the group-relative reward baseline. Specifically, let
$a^{2{:}G}\overset{\mathrm{iid}}{\sim}\pi_{\theta_{\mathrm{old}}}(\cdot\mid q)$
denote the policy-sampled rollouts used in the current GRPO step, and let
$r^{\log}=r(a^{\log})$ and $r^j=r(a^j)$.
Without propensity weighting, the anchored baseline is
$
\bar r_{\mathrm{anc}}
=
\frac{1}{G}r^{\log}
+
\frac{G-1}{G}\bar r_{\mathrm{roll}},
\;\;
\bar r_{\mathrm{roll}}
=
\frac{1}{G-1}\sum_{j=2}^{G} r^j .
$
Taking expectation over the policy-sampled rollouts gives
$
\mathbb{E}_{a^{2{:}G}\sim \pi_{\theta_{\mathrm{old}}}}
[\bar r_{\mathrm{anc}}]
=
\frac{1}{G}r^{\log}
+
\frac{G-1}{G}V^{\pi_{\theta_{\mathrm{old}}}}(q).
$
Thus, compared with the standard GRPO baseline, the group baseline is
pulled toward the realized logged reward by a factor $1/G$.

This calibration acts asymmetrically across feedback types. 
A positive-response anchor, with high $r^{\log}$, raises the baseline and thus reduces the advantage of rollouts that merely replicate the exposed item. 
A no-response anchor, with lower $r^{\log}$, lowers the baseline and avoids over-penalizing alternative rollouts under ambiguous feedback. 
In both cases, the anchor influences learning only through baseline calibration, not through a direct off-policy surrogate term.

\vspace{-0.3cm}
\section{Experiments}
\vspace{-0.2cm}
In this section, we address the following research questions:
\setlength{\fboxsep}{0pt}\colorbox[HTML]{DAE8FC}{\textbf{RQ1}}. Does our update method improve recommendation accuracy under continual updates from policy-shaped contextual bandit feedback?
\setlength{\fboxsep}{0pt}\colorbox[HTML]{DAE8FC}{\textbf{RQ2}}. Does our framework mitigate popularity bias and improve distributional generalization under continual updates?
\setlength{\fboxsep}{0pt}\colorbox[HTML]{DAE8FC}{\textbf{RQ3}}. Does logged-action anchoring improve the stability and effectiveness of vanilla GRPO updates under contextual bandit feedback?
\setlength{\fboxsep}{0pt}\colorbox[HTML]{DAE8FC}{\textbf{RQ4}}. Is asymmetric treatment of positive-and no-response feedback necessary?
\setlength{\fboxsep}{0pt}\colorbox[HTML]{DAE8FC}{\textbf{RQ5}}. Does our method increase CTR over baselines in online A/B tests?

\vspace{-0.2cm}
\subsection{Datasets} \label{subsection:datasets}
\vspace{-0.2cm}
To validate the effectiveness of our proposed framework, we conduct experiments across five distinct subsets of the \textbf{Amazon Reviews 2023} dataset~\citep{hou2024bridging} and the \textbf{MovieLens} benchmark~\citep{harper2015movielens}. 
These datasets encompass a wide array of recommendation contexts, ranging from e-commerce retail interactions to cinematic consumption patterns; comprehensive data distributions and statistics are available in Appendix~\ref{app:dataset_overall}.

\vspace{-0.2cm}
\subsection{Baselines and Evaluation Settings} \label{subsection:setup}
\vspace{-0.2cm}
We evaluate our approach using multiple SOTA LLM-Rec backbones: \textbf{Rec-SAVER}, \textbf{Rec-R1}, \textbf{GDPO}, and \textbf{$\text{G}^2$RPO}. 
Initially, every backbone is trained on a supervised next-item prediction task using its original methodology to ensure specialized learning from user interaction histories. 
We then simulate post-deployment updates by refining these trained models with subsequently collected contextual-bandit feedback: the deployed logging policy first exposes recommendations, and the resulting positive- and no-response signals are used as partial feedback (see Appendix~\ref{app:dataset_overall} for details).
Our study compares several update strategies, including \textbf{DEALRec}, \textbf{RL-Rec}, \textbf{UL-Rec}, and our proposed method, across all backbones. 
The final recommendation performance is measured on a chronologically distinct test set for every backbone-update pair. 
Full implementation details are provided in Appendix~\ref{sec:implementation_details}.

\vspace{-0.2cm}
\begin{table}[] \scriptsize
\centering
\caption{Accuracy comparison on \textit{\textbf{Amazon}} and \textit{\textbf{MovieLens}}; the best and second-best results are marked in \textbf{bold} and \underline{underline}, respectively.
}
\setlength{\tabcolsep}{0.8pt}
\renewcommand{\arraystretch}{1.0}
\label{tab:result_1}
\begin{tabular}{c|cl|ccc|ccc|ccc|ccc|ccc}
\toprule
\multirow{3}{*}{\textbf{\begin{tabular}[c]{@{}c@{}}LLM-Rec\\ Backbone\end{tabular}}} & \multicolumn{2}{c|}{\textbf{\begin{tabular}[c]{@{}c@{}}Update\\ Method\end{tabular}}} & \multicolumn{3}{c|}{\textbf{No-Update}}                                 & \multicolumn{3}{c|}{\textbf{DEALRec Update}}                                       & \multicolumn{3}{c|}{\textbf{RL-Rec Update}}                                      & \multicolumn{3}{c|}{\textbf{UL-Rec Update}}                                       & \multicolumn{3}{c}{\textbf{Ours}}                                       \\ \cline{2-18} 
                                                                                     & \multicolumn{2}{c|}{\textbf{Metric}}                                                  & \multirow{2}{*}{HR@1} & \multirow{2}{*}{HR@5} & \multirow{2}{*}{N@5} & \multirow{2}{*}{HR@1} & \multirow{2}{*}{HR@5} & \multirow{2}{*}{N@5} & \multirow{2}{*}{HR@1} & \multirow{2}{*}{HR@5} & \multirow{2}{*}{N@5} & \multirow{2}{*}{HR@1} & \multirow{2}{*}{HR@5} & \multirow{2}{*}{N@5} & \multirow{2}{*}{HR@1} & \multirow{2}{*}{HR@5} & \multirow{2}{*}{N@5} \\ \cline{2-3}
                                                                                     & \multicolumn{2}{c|}{\textbf{Dataset Domain}}                                          &                       &                       &                         &                       &                       &                         &                       &                       &                         &                       &                       &                         &                       &                       &                         \\ \hline
\multirow{6}{*}{\textbf{Rec-SAVER}}                                                  & \multicolumn{2}{c|}{\textbf{Fashion}}                                                 &  12.72                     &39.60                       &26.92                         & 17.04                      &34.37                       &26.62                         &\underline{19.95}                       &\underline{44.69}                       &\underline{33.95}                         &  14.09                     &  29.52                     &23.00                         &\textbf{20.48}                       &\textbf{51.45}                       &\textbf{36.90}                         \\
                                                                                     & \multicolumn{2}{c|}{\textbf{Grocery}}                                                &10.03                       &\underline{27.92}                       &\underline{19.48}                         &10.04                       &23.34                       &17.49                         &9.77                       &24.30                       &18.04                         &\underline{10.13}                       &18.88                       &15.09                         &\textbf{10.50}                       &\textbf{31.72}                       &\textbf{21.84}                         \\
                                                                                     & \multicolumn{2}{c|}{\textbf{Health}}                                                  &4.51                       &\underline{18.29}                       &11.47                         & 5.10                      &15.26                       &10.64                         &\textbf{8.44}                       &17.53                       &\underline{13.60}                         &3.42                       &11.36                       &7.71                         &6.20                       &\textbf{21.81}                       &\textbf{14.30}                         \\
                                                                                     & \multicolumn{2}{c|}{\textbf{Clothing}}                                                 &8.09                       &26.82                       &17.70                         &7.58                       &23.40                       &16.32                         &    \textbf{9.07}                   &26.30                       &\underline{18.59}                         &\underline{8.98}                       &20.81                       &15.55                         &8.55                       &\textbf{32.76}                       &\textbf{21.24}                         \\
                                                                                     & \multicolumn{2}{c|}{\textbf{CDs}}                                                     & 10.67                      &35.22                       &23.18                         &12.11                       &32.22                       &23.11                         &\underline{19.71}                       &\underline{37.20}                       &\underline{29.43}                         &15.56                       &30.26                       &23.82                         &\textbf{20.59}                       &\textbf{42.46}                       &\textbf{32.12}                         \\
                                                                                     & \multicolumn{2}{c|}{\textbf{MovieLens}}                                               &2.23                       &\underline{11.88}                       &\underline{7.09}                         &\underline{3.65}                       &9.93                       &7.09                         &\textbf{3.73}                       &9.54                       &6.92                         &2.24                       &5.63                       &4.15                         &2.98                       &\textbf{14.63}                       &\textbf{8.88}                         \\ \hline
\multirow{6}{*}{\textbf{Rec-R1}}                                                     & \multicolumn{2}{c|}{\textbf{Fashion}}                                                 & 10.16                      &34.88                       &22.69                         &13.78                       &34.93                       &25.84                         &16.42                       &\underline{37.57}                       &\underline{28.41}                         &\underline{19.25}                       &30.24                       &25.57                         &\textbf{21.03}                       &\textbf{51.71}                       &\textbf{37.97}                         \\
                                                                                     & \multicolumn{2}{c|}{\textbf{Grocery}}                                                &10.23                       &\underline{25.93}                       &\underline{18.61}                         &\textbf{10.32}                       &21.45                       &16.61                         &8.88                       &21.18                       &15.85                         &8.79                       &17.89                       &14.02                         &\underline{9.47}                       & \textbf{27.86}                      &\textbf{19.29}                         \\
                                                                                     & \multicolumn{2}{c|}{\textbf{Clothing}}                                                 & \underline{10.46}                      &\underline{29.65}                       &\underline{19.90}                         &6.57                       &21.67                       &14.92                         &9.57                       &22.52                       &16.81                         &9.41                       &19.69                       &15.26                         &\textbf{13.33}                       &\textbf{31.00}                       &\textbf{22.62}                         \\
                                                                                     & \multicolumn{2}{c|}{\textbf{Health}}                                                  & 2.92                      &9.69                       &6.45                         &4.02                       &15.72                       &10.44                         &\underline{5.80}                       &\underline{16.22}                       &\underline{11.38}                         &\textbf{5.80}                       &11.89                       &9.28                         &4.94                       &\textbf{21.93}                       &\textbf{13.64}                         \\
                                                                                     & \multicolumn{2}{c|}{\textbf{CDs}}                                                     & 8.71                      &\underline{32.59}                       &21.03                         &16.23                       &29.88                       &23.96                         &\underline{18.66}                       &31.10                       &\underline{25.83}                         &18.12                       &29.25                       &24.62                         &\textbf{20.02}                       &\textbf{38.03}                       &\textbf{29.89}                         \\
                                                                                     & \multicolumn{2}{c|}{\textbf{MovieLens}}                                               &1.84                       &\underline{9.73}                       &\underline{5.75}                         &\underline{2.73}                       &8.01                       &5.66                         &2.72                       &8.02                       &5.68                         &1.77                       & 4.28                      &3.20                         &\textbf{3.63}                       &\textbf{11.17}                       &\textbf{7.50}                         \\ \hline
\multirow{6}{*}{\textbf{GDPO}}                                                       & \multicolumn{2}{c|}{\textbf{Fashion}}                                                 & 15.09                      &39.71                       &28.33                         &16.84                       &\underline{41.29}                       &29.96                         &\underline{23.61}                       &38.37                       &\underline{31.98}                         &19.33                       &32.84                       &27.16                         &\textbf{32.80}                       &\textbf{43.73}                       &\textbf{39.26}                         \\
                                                                                     & \multicolumn{2}{c|}{\textbf{Grocery}}                                                &12.01                       &\underline{21.92}                       &\underline{17.46}                         &10.65                       &19.52                       &15.71                         &\underline{12.34}                       &19.96                       &16.74                         &10.32                       &17.38                       &14.45                         &\textbf{12.51}                       &\textbf{25.39}                       &\textbf{19.54}                         \\
                                                                                     & \multicolumn{2}{c|}{\textbf{Clothing}}                                                 &11.40                       &\underline{22.63}                       &\underline{17.42}                         &8.80                       &19.79                       &14.99                         &10.86                       &20.49                       &16.33                         &\underline{11.80}                       &19.65                       &16.39                         &\textbf{12.54}                       &\textbf{26.28}                       &\textbf{20.15}                         \\
                                                                                     & \multicolumn{2}{c|}{\textbf{Health}}                                                  &\underline{7.57}                       &\underline{17.27}                       &\underline{12.97}                         &7.52                       &15.19                       &11.87                         &6.27                       &15.89                       &11.57                         &6.87                       &13.01                       &10.41                         &\textbf{8.30}                       &\textbf{18.22}                       &\textbf{13.67}                         \\
                                                                                     & \multicolumn{2}{c|}{\textbf{CDs}}                                                     &15.68                       &\underline{31.54}                       &24.27                         &18.41                       &29.13                       &24.50                         &\underline{21.00}                       &30.64                       &\underline{26.63}                         &20.70                       &29.14                       &25.60                         &\textbf{21.07}                       &\textbf{37.22}                       &\textbf{29.84}                         \\
                                                                                     & \multicolumn{2}{c|}{\textbf{MovieLens}}                                               &1.32                       &\underline{7.45}                       &4.56                         &2.42                       &6.54                       &4.73                         &\underline{2.74}                       &7.95                       &\underline{5.67}                         &1.64                       &3.75                       &2.83                         &\textbf{2.75}                       &\textbf{9.95}                       &\textbf{6.60}                         \\ \hline
\multirow{6}{*}{\textbf{$\text{G}^2$RPO}}                                                       & \multicolumn{2}{c|}{\textbf{Fashion}}                                                 &22.10                       &\underline{45.97}                       &34.70                         &21.23                       &37.01                       &30.11                         &\textbf{28.96}                       &42.50                       &\underline{36.88}                         & 15.50                      &30.76                       &24.24                         &25.46                       &\textbf{49.87}                       &\textbf{38.54}                         \\
                                                                                     & \multicolumn{2}{c|}{\textbf{Grocery}}                                                & 11.17                      &\underline{28.60}                       &\underline{20.43}                         &9.84                       &22.79                       &17.18                         &\underline{11.22}                       &21.18                       &16.98                         &10.23                       &19.01                       &15.30                         &\textbf{13.13}                       &\textbf{28.70}                       &\textbf{21.57}                         \\
                                                                                     & \multicolumn{2}{c|}{\textbf{Clothing}}                                                 & 10.46                      &\underline{29.65}                       &19.90                         &7.95                       &21.86                       &15.76                         &\textbf{15.32}                       &24.49                       &\underline{20.64}                         &9.85                       &20.50                       &16.00                         &\underline{11.23}                       &\textbf{33.50}                       &\textbf{22.59}                         \\
                                                                                     & \multicolumn{2}{c|}{\textbf{Health}}                                                  &7.18                       &\underline{19.74}                       &\underline{13.89}                         &5.85                       &16.33                       &11.79                         &\underline{9.90}                       &16.18                       &13.55                         &6.99                       &13.28                       &10.57                         &\textbf{10.43}                       &\textbf{20.39}                       &\textbf{15.84}                         \\
                                                                                     & \multicolumn{2}{c|}{\textbf{CDs}}                                                     & 17.58                      &\underline{36.54}                       &27.79                         &16.74                       &31.73                       &25.26                         &\underline{23.94}                       &36.21                       &\underline{30.99}                         &17.96                       &30.92                       & 25.41                        &\textbf{24.85}                       &\textbf{39.58}                       &\textbf{32.89}                         \\
                                                                                     & \multicolumn{2}{c|}{\textbf{MovieLens}}                                               & 3.29                      &\underline{11.72}                       &\underline{7.64}                         &\underline{3.32}                       &7.97                       &5.87                         &2.67                       &7.42                       &5.29                         & 2.85                      &7.95                       &5.65                         &\textbf{4.30}                       &\textbf{11.97}                       &\textbf{8.36}                         \\ \bottomrule
\end{tabular}
\end{table}

\vspace{-0.1cm}
\subsection{Performance Evaluation (RQ1--2)}
\vspace{-0.1cm}
\paragraph{(1) Recommendation Accuracy (RQ1).}

We first evaluate whether our update method improves recommendation accuracy under continual updates from policy-shaped contextual bandit feedback. 
To this end, we report standard top-K recommendation metrics, including HR@1, HR@5, and NDCG@5. 
Overall, our update method achieves the best or highly competitive accuracy on the vast majority of domain--backbone pairs. 
In particular, it consistently outperforms \emph{no update} as well as naive feedback-driven baselines such as DEALRec and UL-Rec updates, indicating that our method can exploit contextual bandit feedback more effectively for improving future recommendation quality. 
At the same time, \emph{no update} still outperforms several alternative update strategies in many settings, underscoring that update rules which fail to properly account for the structure of contextual-bandit logs can degrade recommendation quality despite using additional feedback.

\vspace{-0.3cm}
\paragraph{(2) Popularity Bias (RQ2).}
We next evaluate whether our update method mitigates prior-policy reinforcement bias while maintaining strong recommendation quality.
In recommender systems, optimizing only for accuracy on policy-shaped feedback can make the model repeatedly imitate items favored by the prior policy, resulting in repetitive recommendations and reduced exposure to alternative or long-tail items.
To quantify this effect, we measure popularity-based diversity.

\begin{wraptable}[14]{r}{8.5cm}
\tiny
\setlength{\tabcolsep}{2.2pt}
\vspace{-2.3em}
\caption{
Comparison of recommendation diversity across baselines.
Each cell reports Div@1 / Div@5.
}
\centering
\resizebox{8.5cm}{!}{
\renewcommand{\arraystretch}{0.7}
\begin{tabular}{@{}llcccc@{}}
\toprule
\textbf{Domain} & \textbf{Backbone} 
& \textbf{DEALRec} 
& \textbf{RLRec} 
& \textbf{UL-Rec} 
& \textbf{Ours} \\
\midrule

\multirow{2}{*}{\textbf{Fashion}} 
& GRPO 
& 23.73 / 27.12 
& 25.53 / \underline{28.51} 
& \underline{27.88} / 26.95 
& \textbf{31.07} / \textbf{32.44} \\
& GDPO 
& 26.61 / 27.43 
& \underline{30.72} / \underline{30.25} 
& 28.50 / 28.71 
& \textbf{37.07} / \textbf{35.22} \\

\midrule
\multirow{2}{*}{\textbf{Grocery}} 
& GRPO 
& \underline{27.27} / 27.57 
& 26.71 / \underline{27.71} 
& 26.31 / 26.83 
& \textbf{28.94} / \textbf{28.52} \\
& GDPO 
& 27.01 / 27.57 
& \underline{29.38} / \underline{27.97} 
& 28.78 / 27.07 
& \textbf{30.60} / \textbf{28.70} \\

\midrule
\multirow{2}{*}{\textbf{Clothing}} 
& GRPO 
& 28.06 / 29.56 
& \underline{29.77} / \underline{29.89} 
& 29.60 / 29.37 
& \textbf{31.53} / \textbf{30.79} \\
& GDPO 
& 29.63 / 29.89 
& \underline{30.53} / \underline{30.27} 
& 30.02 / 30.03 
& \textbf{31.31} / \textbf{31.33} \\

\midrule
\multirow{2}{*}{\textbf{Health}} 
& GRPO 
& \underline{10.03} / 10.55 
& \textbf{10.70} / \underline{10.73} 
& 9.82 / 9.98 
& 9.55 / \textbf{11.17} \\
& GDPO 
& \underline{11.34} / \underline{11.36} 
& 11.06 / 10.97 
& 10.56 / 10.81 
& \textbf{12.19} / \textbf{12.08} \\

\midrule
\multirow{2}{*}{\textbf{CDs}} 
& GRPO 
& 20.23 / 22.34 
& \underline{21.95} / \underline{23.38} 
& 21.56 / 23.24 
& \textbf{24.42} / \textbf{25.24} \\
& GDPO 
& \textbf{24.85} / 23.11 
& 23.60 / 24.17 
& \underline{24.07} / \underline{24.37} 
& 22.32 / \textbf{26.00} \\

\midrule
\multirow{2}{*}{\textbf{MovieLens}} 
& GRPO 
& \underline{3.66} / 3.64 
& \textbf{3.89} / \underline{3.67} 
& 3.35 / 3.40 
& 3.45 / \textbf{3.75} \\
& GDPO 
& 3.61 / 3.56 
& 3.43 / \textbf{3.65} 
& \underline{3.91} / 3.38 
& \textbf{3.94} / \underline{3.60} \\

\bottomrule
\end{tabular}
}
\label{tab:diversity}
\end{wraptable}

Let \(p(x)\) denote the normalized popularity of item \(x\) estimated from the training data, and let \(R_u@K\) be the top-\(K\) recommendation list for user \(u\). 
We define Diversity@\(K\) as
$
\mathrm{Diversity@}K(u)
=
\frac{1}{K}
\sum_{x \in R_u@K}
-\log p(x).$
For comparability, we min--max normalize this score to \([0,1]\) using dataset-level item popularity, multiply it by \(100\) to obtain a \(0\)--\(100\) scale, and report the average over all users.
A value close to 100 indicates stronger long-tail emphasis, whereas a value close to 0 indicates higher popularity bias. 
Our results show that our update method yields higher diversity than competing LLM-based update strategies in most cases, indicating that its accuracy gains are not obtained by simply over-reinforcing a narrow set of already popular items.
Overall, our method achieves a better balance between recommendation accuracy and popularity-aware diversity under continual updates from policy-shaped deployment feedback.

\vspace{-0.3cm}
\subsection{Ablation Study (RQ3--4)}
\vspace{-0.1cm}

\begin{wrapfigure}[8]{r}{0.57\textwidth}  
\vspace{-0.7em} 
  \centering
  \includegraphics[width=0.57\textwidth]{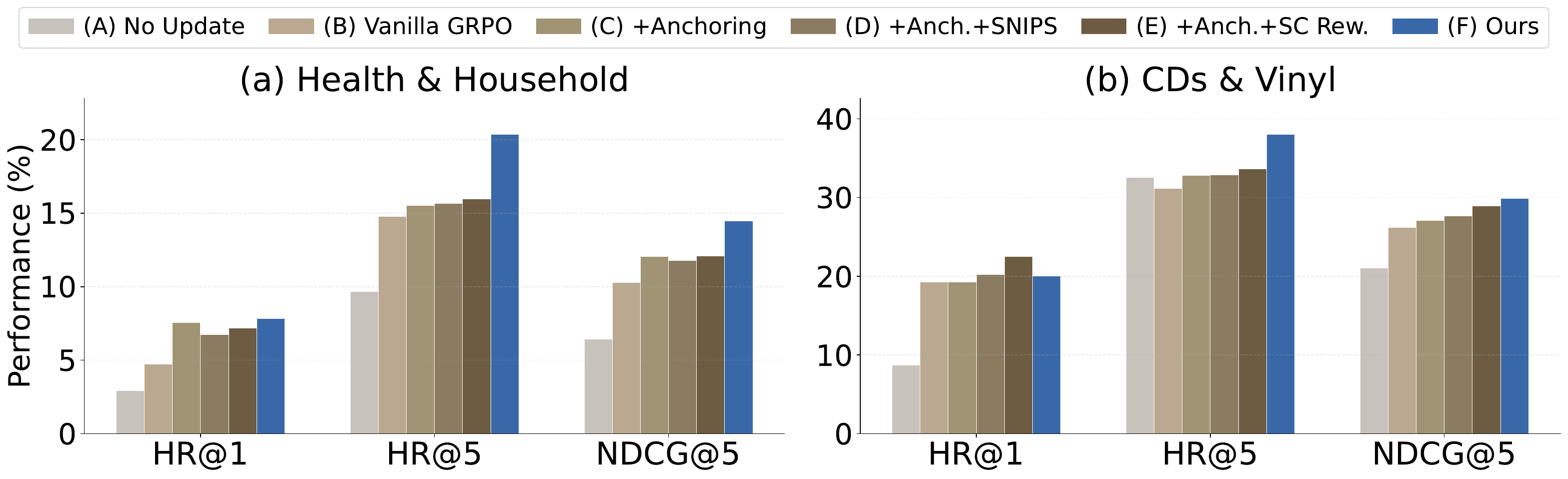}
\caption{
Comparison of in-domain recommendation performance across model variants.
}
  \label{fig:ablation}
\end{wrapfigure}

As shown in Fig.~\ref{fig:ablation}, we conduct a component-wise ablation study, using \textsc{Rec-R1} as the backbone. 
We compare six variants:
\setlength{\fboxsep}{0pt}\colorbox[HTML]{DAE8FC}{(A) \textbf{No Update}};
\setlength{\fboxsep}{0pt}\colorbox[HTML]{DAE8FC}{(B) \textbf{Vanilla GRPO}} (uniformly treating positive- and no-response feedback with fixed rewards);
\setlength{\fboxsep}{0pt}\colorbox[HTML]{DAE8FC}{(C) \textbf{+Anchoring Only}} (inserting the logged action as an anchor into each rollout group);
\setlength{\fboxsep}{0pt}\colorbox[HTML]{DAE8FC}{(D) \textbf{+Anchoring+SNIPS}} (adding importance correction for positive-response samples);
\setlength{\fboxsep}{0pt}\colorbox[HTML]{DAE8FC}{(E) \textbf{+Anchoring+SC Reward}} (adding self-certainty reward for no-response samples);
\setlength{\fboxsep}{0pt}\colorbox[HTML]{DAE8FC}{(F) \textbf{Ours}} (our full method: anchoring + SNIPS + SC reward).

\noindent\textbf{(1) Naively applying GRPO to contextual bandit feedback is not sufficient.}
Compared with \textbf{No Update} (A), \textbf{Vanilla GRPO} (B) often improves performance, but the gains are inconsistent and several domains show little benefit or even degradation. 
This indicates that simply applying GRPO to policy-shaped contextual bandit feedback does not reliably lead to better generalization.

\noindent\textbf{(2) Logged-action anchoring is the key stabilizing design.}
Comparing \textbf{Vanilla GRPO} (B) and \textbf{+Anchoring Only} (C), we observe that anchoring improves recommendation performance on most domains, showing that the logged action serves as a useful reference point rather than a label to imitate directly. 
This supports our answer to \textbf{RQ3}: inserting the logged action into each rollout group makes GRPO updates more effective under contextual bandit feedback.

\noindent\textbf{(3) Response feedback benefits from off-policy correction.}
Comparing \textbf{+Anchoring Only} (C) and \textbf{+Anchoring+SNIPS} (D), SNIPS yields additional gains on most domains.
This suggests that response feedback is informative but still off-policy, making importance-based correction beneficial.

\noindent\textbf{(4) No-response should not be treated as a hard negative.}
Comparing \textbf{+Anchoring Only} (C) and \textbf{+Anchoring+SC Reward} (E), we find that self-certainty reward improves performance in many cases. 
This is because no-response is inherently ambiguous: it does not necessarily mean that the exposed item was truly bad. 
Rather than giving a fixed negative penalty, self-certainty introduces a softer confidence-aware signal, which reduces overly aggressive suppression and improves generalization.

\noindent\textbf{(5) The full design is complementary and most effective.}
Finally, \textbf{Ours} (F) consistently outperforms the partial variants (B--E), showing that anchoring, SNIPS, and self-certainty reward play complementary roles. 
Together, these results confirm that effective continual updating under contextual bandit feedback requires both a stable anchor-based relative learning signal and asymmetric handling of positive- and no-response feedback.

\begin{wraptable}[6]{r}{8.2cm}
\scriptsize
\setlength{\tabcolsep}{1.2pt}
\vspace{-2.8em}
\caption{
Monthly Online A/B Test Results
}
\centering
\resizebox{8.2cm}{!}{
\setlength{\tabcolsep}{1.2pt}
\renewcommand{\arraystretch}{1.08}
\begin{tabular}{c|ccc|ccc|ccc}
\toprule
\multirow{2}{*}{\textbf{Model}} 
& \multicolumn{3}{c|}{\textbf{Feb.}} 
& \multicolumn{3}{c|}{\textbf{Mar.}} 
& \multicolumn{3}{c}{\textbf{Apr.}} \\
\cmidrule(lr){2-4} \cmidrule(lr){5-7} \cmidrule(lr){8-10}
& \textbf{Imp.} & \textbf{Clk.} & \textbf{CTR}
& \textbf{Imp.} & \textbf{Clk.} & \textbf{CTR}
& \textbf{Imp.} & \textbf{Clk.} & \textbf{CTR} \\
\midrule
\textbf{Popolarity} 
& 106.58K & 3.26K & 3.06\%
& 100.33K & 3.48K & 3.47\%
& 95.05K & 3.13K & 3.30\% \\

\textbf{ML model} 
& 136.08K & 9.70K & 7.13\%
& 128.92K & 9.12K & 7.07\%
& 122.81K & 8.71K & 7.09\% \\

\textbf{GRPO-Update} 
& 184.56K & 6.69K & 3.63\%
& 180.17K & 8.16K & 4.53\%
& 170.74K & 7.18K & 4.20\% \\

\cellcolor[HTML]{EFEFEF}\textbf{Ours} 
& \cellcolor[HTML]{EFEFEF}323.09K 
& \cellcolor[HTML]{EFEFEF}23.36K 
& \cellcolor[HTML]{EFEFEF}\textbf{7.23\%}
& \cellcolor[HTML]{EFEFEF}306.67K 
& \cellcolor[HTML]{EFEFEF}22.94K 
& \cellcolor[HTML]{EFEFEF}\textbf{7.48\%}
& \cellcolor[HTML]{EFEFEF}293.69K 
& \cellcolor[HTML]{EFEFEF}22.37K 
& \cellcolor[HTML]{EFEFEF}\textbf{7.62\%} \\
\bottomrule
\end{tabular}
}
\label{tab:monthly_online_ab_main}
\end{wraptable}

\vspace{-0.3cm}
\subsection{Online A/B Test (RQ5)} \label{subsection: online}
\vspace{-0.2cm}

We validated our method via an online A/B test on a production recommender system
(Feb 2026--Apr 2026), comparing it against a rolling-window popularity model,
a periodically retrained ML recommender, and a GRPO-based continual update baseline.
Both the GRPO-Update baseline and our method use Rec-R1 as the backbone.
Although traffic shares differed, this imbalance reflects a production-oriented
allocation policy: verified LLM-Rec variants with stronger offline and preliminary
online performance received larger traffic shares to improve overall service
quality. User assignment and metric aggregation were conducted at the user level,
reducing the risk of inflated significance from repeated impressions.
As shown in Table~\ref{tab:monthly_online_ab_main}, our method improves across all
three months, increasing from 7.23\% CTR in February to 7.48\% in March and
7.62\% in April.
It outperforms ML-Retrain in every month: 7.23\% vs. 7.13\% in February,
7.48\% vs. 7.07\% in March, and 7.62\% vs. 7.09\% in April.
By contrast, ML-Retrain remains nearly flat, GRPO-Update improves in March but
drops in April, and the popularity-based model fluctuates over time.
Further details are provided in Appendix~\ref{sec:ab_test_detail}.

\vspace{-0.3cm}
\vspace{-0.1cm}
\section{Conclusion}
\vspace{-0.3cm}
We introduced \textbf{Anchored Bandit Policy Optimization (ABPO)} for continual LLM-Rec updates from policy-shaped deployment feedback.
ABPO anchors updates on logged exposures, corrects off-policy bias with SNIPS, and tempers ambiguous no-responses with self-certainty, leading to a more consistent upward CTR trend in production than other update strategies.

\clearpage
\bibliographystyle{unsrtnat}
\bibliography{neurips_2026}
\clearpage
\appendix
\section{Dataset} \label{app:dataset_overall}
We conduct experiments on five \textbf{Amazon Review 2023}\footnote{{\color[HTML]{0037D7}\url{https://amazon-reviews-2023.github.io}}}~\citep{hou2024bridging} domains and the \textbf{MovieLens} dataset\footnote{{\color[HTML]{0037D7}\url{https://grouplens.org/datasets/movielens}}}~\citep{harper2015movielens}, which together cover heterogeneous recommendation scenarios ranging from e-commerce purchases to movie consumption, as summarized in Table~\ref{tab:data_stat}. Specifically, the Amazon Review dataset includes five practically relevant domains: (1) \textit{Amazon \textbf{Fashion}}, (2) \textit{\textbf{Clothing} Shoes and Jewelry}, (3) \textit{\textbf{Grocery} and Gourmet Food}, (4) \textit{\textbf{Health} and Household}, and (5) \textit{\textbf{CDs} and Vinyl}. In addition, MovieLens 32M (\textbf{\textit{ml-32m}}) serves as a standard large-scale benchmark for movie recommendation.

\begin{wraptable}[8]{r}{7.4cm}
\small 
\setlength{\tabcolsep}{1.8pt}
\vspace{-1.6em}
\caption{
Dataset Statistics
}

\centering
    \resizebox{7.4cm}{!}{
    \setlength{\tabcolsep}{1.3pt}
    \renewcommand{\arraystretch}{1.1}
\begin{tabular}{c|c|ccccc}
\toprule
\textbf{Dataset} &
  \textbf{Domain} &
  \textbf{\#Users} &
  \textbf{\#Items} &
  \textbf{\begin{tabular}[c]{@{}c@{}}Avg \\ Length\end{tabular}} &
  \textbf{\begin{tabular}[c]{@{}c@{}}Avg Item \\ Purchase\end{tabular}} &
  \textbf{Sparsity} \\ \hline

\multirow{5}{*}{\textbf{Amazon Review}}
& \textbf{Fashion}   & 11,028 & 59,004  & 6.54  & 1.23  & 0.9994 \\
& \textbf{Clothing}  & 13,766 & 107,353 & 11.83 & 1.65  & 0.9994 \\
& \textbf{Grocery}   & 11,334 & 6,149   & 10.76 & 21.02 & 0.9926 \\
& \textbf{Health}    & 21,152 & 74,784  & 8.49  & 2.50  & 0.9994 \\ 
& \textbf{CDs}       & 17,035 & 83,774  & 13.71 & 1.97  & 0.9993 \\ \hline

\multirow{1}{*}{\textbf{MovieLens}}
& \textbf{ml-32m}    & 10,623 & 12,259  & 48.96 & 42.43 & 0.9960 \\ 

\bottomrule
\end{tabular}
}
\label{tab:data_stat}
\end{wraptable}

For each dataset, we assume that every user is associated with a chronological
interaction sequence over timesteps $1:T+N$. Based on this sequence, we
construct three types of data: (1) data for training the initial policy model,
(2) offline contextual bandit logs generated by the learned logging policy,
and (3) a held-out test set. The initial policy is first trained in a standard
next-item prediction manner, using the interaction history up to timestep $T$
to predict the item at timestep $T+1$. Specifically, multiple backbone models,
including Rec-SAVER and Rec-R1, are separately trained on this data and
subsequently used as the initially deployed logging policies
$\pi_{\theta_{\mathrm{init}}}$ for continual updating.

\paragraph{Offline construction of synthetic bandit feedback.}
Definition~2 describes post-deployment contextual bandit feedback in an online
system. Since standard sequential recommendation datasets do not contain real
randomized exposure-response logs, we construct synthetic one-step bandit
feedback from next-item sequences.

For each user sequence, the prefix \(\mathbf{x}_{1:T}\) is treated as the
initial history, and \(x_{T+N}\) is reserved as the held-out evaluation target.
We construct update logs only for
\[
t'\in\{T,\ldots,T+N-2\},
\]
so that the next item used to construct synthetic feedback,
$
x_{t'+1},
$
always belongs to \(\{x_{T+1},\ldots,x_{T+N-1}\}\) and never equals the held-out
evaluation target \(x_{T+N}\).

For each update context \(\mathbf{x}_{1:t'}\), we construct a controlled \textit{M}-item candidate set \(\mathcal C_{t'}\). 
Specifically, we include the next
observed item \(x_{t'+1}\) as the unique latent positive item and uniformly
sample \textit{M}-1 distractor items from the user's non-interacted items:
\[
\mathcal C_{t'}
=
\{x_{t'+1}\}\cup \mathcal N_{t'},
\qquad
|\mathcal N_{t'}|=M-1,
\qquad
|\mathcal C_{t'}|=M.
\]
The distractor set \(\mathcal N_{t'}\) is sampled from items that do not appear
in the user's observed interaction sequence and excludes the held-out
evaluation target \(x_{T+N}\). 
This controlled candidate construction ensures that a positive feedback event is possible for every update context while the held-out evaluation target is never included in any training-time candidate set.
Unless otherwise specified, we use a candidate set size of \textit{M}=200 items.

Given \(\mathcal C_{t'}\), the logging policy scores every candidate item
\(a\in\mathcal C_{t'}\) using the length-normalized LLM likelihood
\[
s_0(a,\mathbf{x}_{1:t'})
=
\frac{1}{|r(a)|}
\sum_j
\log \pi_{\theta_{\mathrm{init}}}
\left(
r_j(a)
\mid
\mathrm{Prompt}(\mathbf{x}_{1:t'},\mathcal C_{t'}), r_{<j}(a)
\right),
\]
where \(r(a)\) denotes the textual representation of item \(a\). These scores
define a candidate-level logging distribution
\[
e_0(a\mid \mathbf{x}_{1:t'})
=
\frac{
\exp(s_0(a,\mathbf{x}_{1:t'})/\tau)
}{
\sum_{b\in\mathcal C_{t'}}
\exp(s_0(b,\mathbf{x}_{1:t'})/\tau)
}.
\]
The logged recommendation is then sampled as
\[
a_{t'}^{\log}
\sim
\mathrm{Categorical}
\left(
e_0(\cdot\mid \mathbf{x}_{1:t'})
\right),
\]
and we store its logging propensity
$
e_0(a_{t'}^{\log}\mid \mathbf{x}_{1:t'}).
$
This propensity is exact within our offline construction because all 200 items
in \(\mathcal C_{t'}\) are scored and the logged item is sampled from the
resulting controlled candidate-level softmax distribution. It is therefore a
candidate-conditional exposure probability, not a raw token likelihood and not
a probability normalized over the full item catalog.

Because real user responses to the sampled exposure are unavailable, we assign
a synthetic bandit feedback label by checking whether the sampled exposed item
matches the next observed item:
\[
y_{t'}
=
\mathbf{1}
\left[
a_{t'}^{\log}=x_{t'+1}
\right].
\]
Thus, \(y_{t'}=1\) means that the logging policy exposed the latent positive
item in the candidate set, while \(y_{t'}=0\) means that it exposed one of the
distractor items. We use \(y_{t'}=0\) as a no-positive-response proxy rather
than an explicit negative preference, since the offline sequence does not
provide direct evidence that the user would dislike the exposed distractor.

The resulting synthetic contextual bandit dataset is
\[
\mathcal D_{\mathrm{bandit}}
=
\left\{
\left(
\mathbf{x}_{1:t'},
a_{t'}^{\log},
y_{t'},
e_0(a_{t'}^{\log}\mid \mathbf{x}_{1:t'})
\right)
\right\}_{t'=T}^{T+N-2}.
\]

Finally, each updated model is evaluated on the held-out next-item prediction
task: given the evaluation context
\[
\mathbf{x}_{1:T+N-1},
\]
the model predicts the held-out item \(x_{T+N}\). Since update feedback labels
are constructed only from next items up to \(x_{T+N-1}\), and
\(x_{T+N}\) is excluded from all training-time candidate sets, the held-out
target is not used for logging, feedback construction, or policy updating.

\section{Reproducibility and Implementation Details} \label{sec:implementation_details}

All experiments were conducted on a cluster of eight NVIDIA H200 GPUs.
To ensure reproducibility, we make our implementation code publicly available through anonymous links.

\noindent$\blacksquare$ \textbf{Code Link}: {\footnotesize\color[HTML]{0037D7}\url{https://anonymous.4open.science/r/ABPO-5CD8}}

\begin{wraptable}[15]{r}{6.0cm}
\tiny
\setlength{\tabcolsep}{1.2pt}
\vspace{-3.5em}
\caption{
Training hyperparameters
}
\centering
\resizebox{6.0cm}{!}{
\renewcommand{\arraystretch}{0.7}
\begin{tabular}{@{}lccc@{}}
\toprule
\multirow{2}{*}{\textbf{Hyperparameter}} 
& \multicolumn{2}{c}{\textbf{Backbone}} 
& \multirow{2}{*}{\textbf{Update}} \\
\cmidrule(lr){2-3}
& \textbf{SFT} & \textbf{RL} & \\
\midrule

Batch size & 8 & 4 & 4 \\
Epochs & 2 & 2 & 1 \\
Max seq. length & 10k & 10k & 12k \\

\midrule
Optimizer 
& \multicolumn{3}{c}{Paged AdamW 8-bit} \\
Learning rate 
& \multicolumn{3}{c}{$5.0 \times 10^{-5}$} \\
Warm-up ratio 
& \multicolumn{3}{c}{0.05} \\
Grad. accum. steps 
& \multicolumn{3}{c}{8} \\
Weight decay 
& \multicolumn{3}{c}{0.01} \\
Grad. clipping 
& \multicolumn{3}{c}{1.0} \\

\midrule
LoRA $r$ / $\alpha$ / dropout 
& \multicolumn{3}{c}{8 / 16 / 0} \\
LoRA modules 
& \multicolumn{3}{c}{$v,q,o$, gate, up, down proj.} \\

\midrule
Precision 
& \multicolumn{3}{c}{bf16} \\
Grad. checkpointing 
& \multicolumn{3}{c}{True} \\
Parallelism 
& \multicolumn{3}{c}{DeepSpeed ZeRO-2} \\
Hardware 
& \multicolumn{3}{c}{8 $\times$ NVIDIA H200 GPUs} \\

\bottomrule
\end{tabular}
}
\label{tab:hyperparams}
\end{wraptable}

\paragraph{Hyperparameters.}
Table~\ref{tab:hyperparams} summarizes the hyperparameter configurations used in our experiments, including the learning rate, batch size, optimizer, warm-up ratio, gradient clipping threshold, weight decay, and training epochs.
We use the same LLM backbone (i.e., \textsc{gemma-3-4b-it}) for all methods to ensure a fair comparison.

Training proceeds in two stages.
First, we obtain the initial deployed policy $\pi_{\theta_{\mathrm{init}}}$ through supervised fine-tuning (SFT) on the supervised recommendation dataset $\mathcal{D}_{\mathrm{sup}}$.
This SFT checkpoint is shared by all RL-based update methods, including GRPO-based baselines and our proposed ABPO.
Second, each method performs continual policy updates using the post-deployment contextual-bandit dataset $\mathcal{D}_{\mathrm{bandit}}$.

For all RL-based update methods, we use the same rollout group size (16), decoding configuration, optimizer, and parameter-efficient fine-tuning setup unless otherwise specified.
Our ABPO uses the same GRPO-style clipped policy objective as the baselines, but differs in how the rollout group is constructed and normalized: the logged exposure is inserted as an anchor, SNIPS is applied to correct the anchor contribution under policy mismatch, and reward components are normalized over the anchor-augmented group.
Parameter-efficient fine-tuning is performed via LoRA~\citep{hu2022lora}, and all implementation details required for reproducibility are provided in Table~\ref{tab:hyperparams}.

\paragraph{Evaluation Protocol}
We adopt a standard leave-one-out evaluation strategy, where the final interaction in each user sequence is reserved for testing, the penultimate interaction for validation, and all earlier interactions for training.
During evaluation, each inference prompt explicitly includes a candidate set of 200 items, consisting of the ground-truth target item and unseen negative candidates.
To better reflect realistic candidate competition, we complement randomly sampled negatives with popularity-based unseen hard negatives.
For all baselines and our model, the user context and candidate set in each prompt are kept identical.
DPO-based baselines require explicit construction of negative samples and were therefore trained with additional negatives.

\section{Baselines}
\label{sec:llm_baseline_details}
\vspace{-0.15cm}

\paragraph{LLM-Rec Update Model.}
We survey and organize prior LLM-Rec update methods relevant to our study, and describe how they are implemented in our experiments.

\noindent$\blacksquare$ \textbf{DEALRec}~\citep{lin2024data} reduces the cost of updating LLM-based recommenders by pruning recommendation data before few-shot fine-tuning. 
It selects influential user sequences using an influence score estimated from a lightweight surrogate recommender and an effort score based on LLM gradients, while applying coverage-enhanced sampling to preserve data diversity. 
We implemented this baseline from scratch.

\noindent$\blacksquare$ \textbf{RLRec}~\citep{yin2025clicks} aligns an LLM-based generative recommender with post-deployment user feedback by converting noisy click logs into preference signals. 
It builds an initial policy via supervised fine-tuning on clicked suggestions, models preference uncertainty with a Gaussian Reward Model (GaRM), and further updates the policy using GRPO with a composite reward to improve user engagement while controlling reward hacking. 
We implemented this baseline from scratch.

\noindent$\blacksquare$ \textbf{ULRec}~\citep{zhang2025leveraging} updates LLM-based recommenders using unpaired post-interaction feedback to improve long-term user satisfaction. 
It decomposes successful trajectories into step-level positive supervision for SFT, and further applies KTO on negative feedback signals such as user disengagement to discourage recommendations that reinforce echo chambers. 
We implemented this baseline from scratch.

\paragraph{LLM-Rec Backbone.}
We summarize the LLM-based baselines used in our experiments and describe their implementations below.

\noindent$\blacksquare$ \textbf{Rec-SAVER}~\citep{tsai2024leveraging} enhances LLM-based recommendation by fine-tuning models with target-item labels and corresponding rationales.
The approach shows that rationale-supervised instruction tuning is effective for both zero-shot transfer and task-specific fine-tuning, including non-overlapping cross-domain settings.
We evaluate this baseline with our own in-house implementation.

\noindent$\blacksquare$ \textbf{Rec-R1}~\citep{lin2025rec} applies GRPO to recommendation, where the reward is based on whether the generated item matches the ground-truth next item. 
We built this baseline using a public GRPO implementation\footnote{{\color[HTML]{0037D7}\url{https://github.com/huggingface/open-r1}}}.

\noindent$\blacksquare$ \textbf{GDPO}~\citep{liu2026gdpo} is a reward-decoupled extension of GRPO for multi-objective RL training. 
Rather than aggregating multiple rewards before normalization, GDPO normalizes each reward component independently and then combines the resulting signals for policy optimization. 
This helps preserve reward-specific variations that may otherwise be obscured by scalar reward aggregation. 
Prior evaluations on tool-use, mathematical reasoning, and coding benchmarks show that GDPO can provide more stable training and stronger overall performance than GRPO. 
Although it was not developed for recommendation, its ability to handle heterogeneous rewards makes it a suitable baseline for LLM-based recommender training. 
We implement GDPO using a public GRPO codebase and adapt it to our recommendation setting.

\noindent$\blacksquare$ \textbf{$\text{G}^2$RPO}~\footnote{{\color[HTML]{0037D7}\url{https://anonymous.4open.science/r/G2RPO-FDB9}}} extends GDPO with accuracy-aware dynamic reward aggregation for LLM-based recommendation.
Instead of uniformly summing normalized reward signals, it computes aggregation weights by solving a simplex-constrained linear program that aligns the final update direction with the accuracy gradient while preventing auxiliary rewards from inducing harmful updates.
This design enables multi-reward optimization over accuracy, diversity, and groundedness without treating all objectives as equally tradeable.

\section{Detailed Online A/B Test Analysis}
\label{sec:ab_test_detail}
We validated our method via an online A/B test on a production recommender system (February 2026--April 2026), comparing it against a rolling-window popularity model, a periodically retrained conventional ML recommender, and a GRPO-based continual update model, with users assigned through stable hash-based randomization. 
Although traffic shares were not uniform, larger shares were allocated to LLM-Rec variants after their effectiveness had been verified, to improve overall production performance, while the randomized user-level assignment kept user populations comparable.
Both the GRPO-based baseline and our method use Rec-R1 as the backbone.
To examine whether each model benefits from successive production updates, we further provide a month-level breakdown from February to April 2026, where each update was customized to reflect the business logic of the production service.
Table~\ref{tab:monthly_online_ab} reports impressions, clicks, CTRs, and pairwise statistical comparisons between our method and each baseline. Statistical significance is computed using a two-sided two-proportion z-test at the impression level.

\begin{table}[t]
\centering
\small
\caption{
Monthly online A/B test results from February to April.
$\Delta$ denotes the CTR difference between our method and each baseline, i.e., Ours $-$ Baseline, in percentage points.
Statistical significance is computed using a two-sided two-proportion z-test.
}
\label{tab:monthly_online_ab}
\resizebox{\linewidth}{!}{
\begin{tabular}{llrrrrrr}
\toprule
Month & Model & Impressions & Clicks & CTR (\%) & $\Delta$ & $z$ & $p$ \\
\midrule
\multirow{4}{*}{February}
& Popularity  & 106,583 & 3,260  & 3.06 & +4.17 & 48.98 & $<0.001$ \\
& ML-Retrain  & 136,084 & 9,702  & 7.13 & +0.10 & 1.20  & 0.232 \\
& GRPO-Update & 184,557 & 6,691  & 3.63 & +3.60 & 52.34 & $<0.001$ \\
& Ours        & 323,085 & 23,357 & 7.23 & --    & --    & -- \\
\midrule
\multirow{4}{*}{March}
& Popularity  & 100,333 & 3,481  & 3.47 & +4.01 & 44.75 & $<0.001$ \\
& ML-Retrain  & 128,922 & 9,121  & 7.07 & +0.40 & 4.67  & $<0.001$ \\
& GRPO-Update & 180,169 & 8,160  & 4.53 & +2.95 & 40.65 & $<0.001$ \\
& Ours        & 306,672 & 22,937 & 7.48 & --    & --    & -- \\
\midrule
\multirow{4}{*}{April}
& Popularity  & 95,047  & 3,133  & 3.30 & +4.32 & 46.75 & $<0.001$ \\
& ML-Retrain  & 122,813 & 8,708  & 7.09 & +0.53 & 5.88  & $<0.001$ \\
& GRPO-Update & 170,743 & 7,179  & 4.20 & +3.41 & 45.92 & $<0.001$ \\
& Ours        & 293,686 & 22,366 & 7.62 & --    & --    & -- \\
\bottomrule
\end{tabular}
}
\end{table}

To further analyze whether the observed improvements reflect a consistent trend rather than a single-month fluctuation, Table~\ref{tab:monthly_trend_test} reports month-to-month CTR changes for each model. In particular, we compare February to March, March to April, and February to April using the same two-sided two-proportion z-test.

\begin{table}[t]
\centering
\small
\caption{
Month-to-month CTR changes in the online A/B test.
Statistical significance is computed using a two-sided two-proportion z-test between the CTRs of the two corresponding months.
}
\label{tab:monthly_trend_test}
\resizebox{\linewidth}{!}{
\begin{tabular}{lrrrrrrrrr}
\toprule
\multirow{2}{*}{\textbf{Model}}
& \multicolumn{3}{c}{\textbf{Feb.$\rightarrow$Mar.}}
& \multicolumn{3}{c}{\textbf{Mar.$\rightarrow$Apr.}}
& \multicolumn{3}{c}{\textbf{Feb.$\rightarrow$Apr.}} \\
\cmidrule(lr){2-4} \cmidrule(lr){5-7} \cmidrule(lr){8-10}
& $\Delta$ CTR & $z$ & $p$
& $\Delta$ CTR & $z$ & $p$
& $\Delta$ CTR & $z$ & $p$ \\
\midrule
\textbf{Popularity}
& +0.41\%p & 5.26 & $<0.001$
& -0.17\%p & -2.12 & 0.034
& +0.24\%p & 3.04 & 0.002 \\

\textbf{ML-Retrain}
& -0.05\%p & -0.55 & 0.584
& +0.02\%p & 0.15 & 0.879
& -0.04\%p & -0.39 & 0.700 \\

\textbf{GRPO-Update}
& +0.90\%p & 13.81 & $<0.001$
& -0.32\%p & -4.70 & $<0.001$
& +0.58\%p & 8.90 & $<0.001$ \\

\textbf{Ours}
& +0.25\%p & 3.80 & $<0.001$
& +0.14\%p & 2.00 & 0.046
& +0.39\%p & 5.78 & $<0.001$ \\
\bottomrule
\end{tabular}
}
\end{table}

Across the three months, our method is the only model that shows a monotonic improvement in CTR: 7.23\% in February, 7.48\% in March, and 7.62\% in April. As shown in Table~\ref{tab:monthly_trend_test}, the improvement is statistically significant in both consecutive intervals: +0.25\%p from February to March ($z=3.80$, $p<0.001$) and +0.14\%p from March to April ($z=2.00$, $p=0.046$). Overall, this yields a statistically significant +0.39\%p gain from February to April ($z=5.78$, $p<0.001$). These results suggest that the improvement is not driven by a single-month fluctuation, but follows a consistent upward trend across successive production update rounds.

The comparison with ML-Retrain provides a conservative reference point because ML-Retrain is the strongest non-LLM continual-update baseline in terms of absolute CTR. In February, our method is numerically higher than ML-Retrain, but the difference is not statistically significant (+0.10\%p, $z=1.20$, $p=0.232$). However, after additional update rounds, the gap becomes statistically significant: +0.40\%p in March ($z=4.67$, $p<0.001$) and +0.53\%p in April ($z=5.88$, $p<0.001$). By contrast, ML-Retrain remains nearly flat over time, with no statistically significant change from February to March ($z=-0.55$, $p=0.584$), from March to April ($z=0.15$, $p=0.879$), or from February to April ($z=-0.39$, $p=0.700$). This suggests that periodic supervised retraining alone does not yield sustained post-deployment improvement under the observed feedback stream.

The popularity-based baseline shows a different pattern. Its CTR increases from 3.06\% in February to 3.47\% in March (+0.41\%p, $z=5.26$, $p<0.001$), but decreases again to 3.30\% in April (-0.17\%p, $z=-2.12$, $p=0.034$). Although its February-to-April change is statistically significant (+0.24\%p, $z=3.04$, $p=0.002$), the month-to-month pattern is not monotonic, and its absolute CTR remains substantially below our method in every month.

GRPO-Update also improves from February to March, increasing from 3.63\% to 4.53\% (+0.90\%p, $z=13.81$, $p<0.001$). However, this improvement is not sustained: its CTR significantly drops from 4.53\% in March to 4.20\% in April (-0.32\%p, $z=-4.70$, $p<0.001$). Although the February-to-April change remains positive overall (+0.58\%p, $z=8.90$, $p<0.001$), the decline after March suggests that naive GRPO-style continual updating may be less stable under policy-shaped contextual-bandit feedback.

Our method significantly outperforms GRPO-Update in every month, with gaps of +3.60\%p in February ($z=52.34$, $p<0.001$), +2.95\%p in March ($z=40.65$, $p<0.001$), and +3.41\%p in April ($z=45.92$, $p<0.001$). This comparison is particularly informative because both methods use the same LLM-based recommender backbone. Therefore, the performance gap mainly reflects the effect of our proposed update design: anchoring the logged exposure, correcting off-policy exposure bias, and tempering ambiguous no-response feedback.

Aggregating all three months, our method obtains 68,660 clicks from 923,443 impressions, corresponding to a CTR of 7.44\%. This outperforms ML-Retrain by +0.34\%p ($z=6.74$, $p<0.001$), GRPO-Update by +3.32\%p ($z=80.08$, $p<0.001$), and the popularity-based model by +4.17\%p ($z=81.13$, $p<0.001$). Together, the aggregate and month-level results suggest that our method not only improves overall CTR, but also enables more stable continual improvement across successive production update rounds.

\vspace{-0.3cm}
\section{Theoretical Analysis}
\label{app:theory}

We formalize the baseline-calibration effect of the logged anchor under
contextual bandit feedback. Throughout this section, expectations are
conditioned on a fixed user context \(q\) and a logged tuple
\((a^{\log},y)\), and we suppress this conditioning for brevity.
The reward function \(r(a)\equiv r(a,q,a^{\log},y)\) may depend on the
logged tuple through the feedback-dependent reward design.

A key distinction in ABPO is that the logged anchor is not optimized as an
on-policy rollout. The logged action \(a^{\log}\) was sampled from the
initial item-level logging distribution \(e_0(\cdot\mid q)\), whereas the
non-anchor rollouts used in the current GRPO update are sampled from the
item-level rollout distribution induced by the previous-step policy:
\[
\rho(\cdot\mid q) := e_{\mathrm{old}}(\cdot\mid q).
\]
The optimized policy \(\pi_\theta\) appears only through the GRPO likelihood
ratio in the final surrogate objective. The analysis below is therefore
with respect to the rollout-generating distribution \(\rho\), not the
optimizing policy \(\pi_\theta\).

For notational simplicity, we write \(e_p(a\mid q)\) for item-level exposure
propensities. In our offline construction, the candidate set \(\mathcal C(q)\)
is deterministically constructed from the context and included in the LLM
prompt. For \(p\in\{0,\mathrm{old}\}\), \(e_p\) is obtained by first computing
a length-normalized average token log-probability \(s_p(a,q)\) for each
candidate item \(a\in\mathcal C(q)\), and then applying a temperature-scaled
softmax:
\[
e_p(a\mid q)
=
\frac{\exp(s_p(a,q)/\tau)}
{\sum_{b\in\mathcal C(q)}\exp(s_p(b,q)/\tau)}.
\]
Thus, \(e_0\) and \(e_{\mathrm{old}}\) are item-level exposure probabilities
over the same candidate item set, not raw token-level likelihoods of the
generated item text.

For clarity, we analyze the centered reward
\(r(a_j)-\bar r_{\mathrm{anc}}^\omega\). In implementation, ABPO further
divides this quantity by a positive weighted standard deviation, as in the
main text. This normalization changes the scale of the advantage but not
the direction of the baseline-induced shift analyzed below.

\paragraph{Notation.}
Let
\[
V^\rho(q)
=
\mathbb{E}_{a\sim \rho(\cdot\mid q)}[r(a)]
\]
denote the expected reward under the rollout distribution. For the \(G-1\)
policy-sampled non-anchor rollouts
\[
a_2,\ldots,a_G \overset{\mathrm{iid}}{\sim} \rho(\cdot\mid q),
\]
define the rollout-only baseline
\begin{equation}
\bar r_{\mathrm{roll}}
=
\frac{1}{G-1}\sum_{j=2}^{G} r(a_j).
\label{eq:rollout_only_baseline}
\end{equation}
The logged anchor contributes only to the group-relative baseline. Given
a nonnegative stopped-gradient anchor weight \(\omega\ge 0\), define the
weighted anchored baseline as
\begin{equation}
\bar r_{\mathrm{anc}}^{\omega}
=
\frac{
\omega r(a^{\log})+\sum_{j=2}^{G}r(a_j)
}{
\omega+G-1
}.
\label{eq:weighted_anchor_baseline_theory}
\end{equation}
The unweighted anchor baseline is recovered by setting \(\omega=1\).
In ABPO, \(\omega\) corresponds to the self-normalized logged-anchor weight
\(\hat w^{\log}\), computed from an item-level propensity ratio and treated
as stop-gradient.

The centered advantages induced by this baseline are
\begin{equation}
A_j^{\omega}
=
r(a_j)-\bar r_{\mathrm{anc}}^{\omega},
\qquad j=2,\ldots,G.
\label{eq:nonanchor_advantage_theory}
\end{equation}
There is no corresponding surrogate advantage for the logged anchor
\(a^{\log}\).

In the implementation, the final advantage additionally uses the weighted
standard deviation
\begin{equation}
\sigma_{\omega}
=
\sqrt{
\frac{
\omega \left(r(a^{\log})-\bar r_{\mathrm{anc}}^{\omega}\right)^2
+
\sum_{j=2}^{G}
\left(r(a_j)-\bar r_{\mathrm{anc}}^{\omega}\right)^2
}{
\omega+G-1
}
+
\epsilon_{\mathrm{std}}
},
\label{eq:weighted_std_theory}
\end{equation}
and applies
\[
\tilde A_j^{\omega}
=
\frac{r(a_j)-\bar r_{\mathrm{anc}}^{\omega}}{\sigma_{\omega}},
\qquad j=2,\ldots,G.
\]
Since \(\sigma_{\omega}>0\), the sign of the centered advantage shift is
preserved.

\subsection{Baseline Calibration}

\begin{lemma}[Rollout-Only Baseline is Centered at the Rollout Distribution]
\label{lem:rollout_baseline_centered}
The expected rollout-only baseline satisfies
\begin{equation}
\mathbb{E}[\bar r_{\mathrm{roll}}]
=
V^\rho(q)
=
\mathbb{E}_{a\sim \rho(\cdot\mid q)}[r(a)].
\label{eq:rollout_baseline_centered}
\end{equation}
\end{lemma}

\begin{proof}
Since \(a_2,\ldots,a_G\) are sampled i.i.d. from \(\rho(\cdot\mid q)\),
each rollout satisfies
\[
\mathbb{E}[r(a_j)]
=
V^\rho(q),
\qquad j=2,\ldots,G.
\]
Therefore, by linearity of expectation,
\[
\mathbb{E}[\bar r_{\mathrm{roll}}]
=
\mathbb{E}\left[
\frac{1}{G-1}\sum_{j=2}^{G}r(a_j)
\right]
=
\frac{1}{G-1}\sum_{j=2}^{G}V^\rho(q)
=
V^\rho(q).
\]
\end{proof}

Lemma~\ref{lem:rollout_baseline_centered} shows that without anchoring,
the group-relative baseline is centered at the average reward of the
current rollout distribution. Thus, although the reward function may depend
on the logged tuple \((a^{\log},y)\), the realized logged reward
\(r(a^{\log})\) does not explicitly enter the rollout-only baseline.

\begin{theorem}[Weighted Anchoring Shifts the Baseline Toward the Logged Reward]
\label{thm:weighted_anchor_calibration}
Condition on a fixed nonnegative stopped-gradient anchor weight \(\omega\).
Then the expected anchored baseline satisfies
\begin{equation}
\mathbb{E}[\bar r_{\mathrm{anc}}^{\omega}\mid \omega]
=
\frac{\omega}{\omega+G-1}r(a^{\log})
+
\frac{G-1}{\omega+G-1}V^\rho(q).
\label{eq:weighted_anchored_expectation}
\end{equation}
Equivalently, if
\[
\alpha_\omega
=
\frac{\omega}{\omega+G-1},
\]
then
\begin{equation}
\mathbb{E}[\bar r_{\mathrm{anc}}^{\omega}\mid \omega]
=
\alpha_\omega r(a^{\log})
+
(1-\alpha_\omega)V^\rho(q).
\label{eq:weighted_convex_combination}
\end{equation}
Thus, the anchored baseline is a convex combination of the realized logged
reward and the rollout distribution's expected reward.

Moreover, define the deviations from the logged reward as
\begin{equation}
\delta_{\mathrm{roll}}
=
\mathbb{E}[\bar r_{\mathrm{roll}}]-r(a^{\log}),
\qquad
\delta_{\mathrm{anc}}^{\omega}
=
\mathbb{E}[\bar r_{\mathrm{anc}}^{\omega}\mid \omega]-r(a^{\log}).
\label{eq:weighted_deviation_defs}
\end{equation}
Then
\begin{equation}
\delta_{\mathrm{anc}}^{\omega}
=
\frac{G-1}{\omega+G-1}\delta_{\mathrm{roll}}.
\label{eq:weighted_deviation_shrinkage}
\end{equation}
In particular, for the unweighted case \(\omega=1\),
\begin{equation}
\delta_{\mathrm{anc}}^{1}
=
\frac{G-1}{G}\delta_{\mathrm{roll}}.
\label{eq:unweighted_deviation_shrinkage}
\end{equation}
\end{theorem}

\begin{proof}
Using Eq.~\eqref{eq:weighted_anchor_baseline_theory},
\begin{align}
\mathbb{E}[\bar r_{\mathrm{anc}}^{\omega}\mid \omega]
&=
\mathbb{E}\left[
\frac{
\omega r(a^{\log})+\sum_{j=2}^{G}r(a_j)
}{
\omega+G-1
}
\,\middle|\, \omega
\right] \notag \\
&=
\frac{
\omega r(a^{\log})
+
\sum_{j=2}^{G}\mathbb{E}[r(a_j)]
}{
\omega+G-1
} \notag \\
&=
\frac{
\omega r(a^{\log})
+
(G-1)V^\rho(q)
}{
\omega+G-1
}.
\end{align}
This proves Eq.~\eqref{eq:weighted_anchored_expectation}. Rewriting with
\(\alpha_\omega=\omega/(\omega+G-1)\) gives
Eq.~\eqref{eq:weighted_convex_combination}.

By Lemma~\ref{lem:rollout_baseline_centered},
\[
\delta_{\mathrm{roll}}
=
V^\rho(q)-r(a^{\log}).
\]
Similarly,
\begin{align}
\delta_{\mathrm{anc}}^{\omega}
&=
\frac{
\omega r(a^{\log})+(G-1)V^\rho(q)
}{
\omega+G-1
}
-
r(a^{\log}) \notag \\
&=
\frac{G-1}{\omega+G-1}
\left(
V^\rho(q)-r(a^{\log})
\right) \notag \\
&=
\frac{G-1}{\omega+G-1}\delta_{\mathrm{roll}}.
\end{align}
Setting \(\omega=1\) yields Eq.~\eqref{eq:unweighted_deviation_shrinkage}.
\end{proof}

Theorem~\ref{thm:weighted_anchor_calibration} formalizes the calibration
role of the logged anchor. The anchor does not create a direct off-policy
surrogate term. Instead, its reward shifts the baseline used to evaluate
the current policy-sampled rollouts. The strength of this shift is
controlled by the stopped-gradient weight \(\omega\).

\subsection{Effect on Non-Anchor Rollout Advantages}

Since the logged anchor is not optimized directly, the relevant quantity
is not the anchor's own advantage, but how the anchored baseline changes
the advantages of the policy-sampled rollouts.

\begin{proposition}[Non-Anchor Advantage Shift]
\label{prop:nonanchor_adv_shift}
Let
\[
A_j^{\mathrm{roll}}
=
r(a_j)-\bar r_{\mathrm{roll}},
\qquad j=2,\ldots,G,
\]
be the rollout-only group-relative advantage. Under the weighted anchored
baseline, the non-anchor advantage satisfies
\begin{equation}
A_j^{\omega}
=
A_j^{\mathrm{roll}}
+
\frac{\omega}{\omega+G-1}
\left(
\bar r_{\mathrm{roll}}-r(a^{\log})
\right),
\qquad j=2,\ldots,G.
\label{eq:nonanchor_advantage_shift}
\end{equation}
Taking expectation over the policy-sampled rollouts gives
\begin{equation}
\mathbb{E}[A_j^{\omega}\mid \omega]
=
\frac{\omega}{\omega+G-1}
\left(
V^\rho(q)-r(a^{\log})
\right),
\qquad j=2,\ldots,G.
\label{eq:expected_nonanchor_advantage_shift}
\end{equation}
\end{proposition}

\begin{proof}
By definition,
\[
A_j^{\omega}
=
r(a_j)-\bar r_{\mathrm{anc}}^{\omega},
\qquad
A_j^{\mathrm{roll}}
=
r(a_j)-\bar r_{\mathrm{roll}}.
\]
Therefore,
\begin{align}
A_j^{\omega}-A_j^{\mathrm{roll}}
&=
\bar r_{\mathrm{roll}}-\bar r_{\mathrm{anc}}^{\omega} \notag \\
&=
\bar r_{\mathrm{roll}}
-
\frac{
\omega r(a^{\log})+\sum_{k=2}^{G}r(a_k)
}{
\omega+G-1
}.
\end{align}
Since
\[
\sum_{k=2}^{G}r(a_k)
=
(G-1)\bar r_{\mathrm{roll}},
\]
we obtain
\begin{align}
A_j^{\omega}-A_j^{\mathrm{roll}}
&=
\bar r_{\mathrm{roll}}
-
\frac{
\omega r(a^{\log})+(G-1)\bar r_{\mathrm{roll}}
}{
\omega+G-1
} \notag \\
&=
\frac{\omega}{\omega+G-1}
\left(
\bar r_{\mathrm{roll}}-r(a^{\log})
\right).
\end{align}
This proves Eq.~\eqref{eq:nonanchor_advantage_shift}. Taking expectation
and using Lemma~\ref{lem:rollout_baseline_centered} gives
Eq.~\eqref{eq:expected_nonanchor_advantage_shift}.
\end{proof}

Proposition~\ref{prop:nonanchor_adv_shift} shows that the logged action
itself is not assigned a surrogate advantage. Instead, the logged reward
shifts the advantages of the non-anchor rollouts through the baseline.

\subsection{Implications for Positive- and No-Response Feedback}

The sign of the baseline shift is determined by the relative ordering
between the logged reward \(r(a^{\log})\) and the rollout distribution's
expected reward \(V^\rho(q)\).

\begin{corollary}[Positive-Response Calibration]
\label{cor:positive_calibration}
Suppose a positive-response logged exposure satisfies
\begin{equation}
r(a^{\log}) > V^\rho(q).
\label{eq:positive_condition_weighted}
\end{equation}
Then the weighted anchored baseline is higher in expectation than the
rollout-only baseline:
\begin{equation}
\mathbb{E}[\bar r_{\mathrm{anc}}^{\omega}\mid \omega]
-
\mathbb{E}[\bar r_{\mathrm{roll}}]
=
\frac{\omega}{\omega+G-1}
\left(
r(a^{\log})-V^\rho(q)
\right)
>0.
\label{eq:positive_baseline_shift_weighted}
\end{equation}
Consequently, the expected centered advantages of the policy-sampled
rollouts are shifted downward by
\begin{equation}
\mathbb{E}[A_j^{\omega}\mid \omega]
=
-
\frac{\omega}{\omega+G-1}
\left(
r(a^{\log})-V^\rho(q)
\right)
<0,
\qquad j=2,\ldots,G.
\label{eq:positive_advantage_shift_weighted}
\end{equation}
\end{corollary}

\begin{proof}
The baseline shift follows directly from
Eq.~\eqref{eq:weighted_anchored_expectation} and
Lemma~\ref{lem:rollout_baseline_centered}:
\[
\mathbb{E}[\bar r_{\mathrm{anc}}^{\omega}\mid \omega]
-
\mathbb{E}[\bar r_{\mathrm{roll}}]
=
\frac{\omega}{\omega+G-1}
\left(
r(a^{\log})-V^\rho(q)
\right).
\]
Under condition~\eqref{eq:positive_condition_weighted}, this quantity is
positive. Eq.~\eqref{eq:positive_advantage_shift_weighted} follows from
Proposition~\ref{prop:nonanchor_adv_shift}.
\end{proof}

Thus, when a positive-response logged item receives a higher reward than
the rollout distribution's average reward, anchoring raises the comparison
baseline. This reduces the relative advantage of policy rollouts that do
not improve upon the positive logged reference, without directly applying
a policy-gradient update to the logged action itself.

\begin{corollary}[No-Response Buffering]
\label{cor:noresponse_buffering}
Suppose a no-response logged exposure satisfies
\begin{equation}
r(a^{\log}) < V^\rho(q).
\label{eq:noresponse_condition_weighted}
\end{equation}
Then the weighted anchored baseline is lower in expectation than the
rollout-only baseline:
\begin{equation}
\mathbb{E}[\bar r_{\mathrm{anc}}^{\omega}\mid \omega]
-
\mathbb{E}[\bar r_{\mathrm{roll}}]
=
\frac{\omega}{\omega+G-1}
\left(
r(a^{\log})-V^\rho(q)
\right)
<0.
\label{eq:noresponse_baseline_shift_weighted}
\end{equation}
Consequently, the expected centered advantages of the policy-sampled
rollouts are shifted upward by
\begin{equation}
\mathbb{E}[A_j^{\omega}\mid \omega]
=
\frac{\omega}{\omega+G-1}
\left(
V^\rho(q)-r(a^{\log})
\right)
>0,
\qquad j=2,\ldots,G.
\label{eq:noresponse_advantage_shift_weighted}
\end{equation}
\end{corollary}

\begin{proof}
The proof is identical to that of
Corollary~\ref{cor:positive_calibration}, with the inequality direction
reversed under condition~\eqref{eq:noresponse_condition_weighted}.
\end{proof}

Thus, when a no-response logged item receives a lower reward than the
rollout distribution's average reward, anchoring lowers the baseline. This
can buffer alternative rollouts from being overly penalized under ambiguous
no-response feedback. Again, the effect is mediated only through the
baseline, not through a direct off-policy surrogate term for the logged
anchor.

\subsection{Anchor-Level Interpretation of SNIPS}
\label{sec:snips_consistency}

We do not use SNIPS to characterize the full ABPO objective. SNIPS is used
only to determine the stopped-gradient coefficient with which the logged
reward contributes to the anchored baseline in
Eq.~\eqref{eq:weighted_anchor_baseline_theory}. The logged anchor is not
assigned a GRPO likelihood ratio or a direct off-policy surrogate advantage.

\paragraph{Item-level propensity ratio.}
Let \(\rho=e_{\mathrm{old}}\) denote the item-level rollout distribution
induced by the previous-step policy, and let \(e_0\) denote the item-level
logging distribution used to sample the logged exposure. For any fixed
feedback type \(y\in\{0,1\}\), let \(f_y(a,q)\) denote a bounded scalar
function of the logged action, such as the feedback-specific logged reward.
Define the item-level importance ratio
\begin{equation}
w(a,q)
=
\frac{
e_{\mathrm{old}}(a\mid q)
}{
e_0(a\mid q)
}.
\label{eq:anchor_ips_ratio}
\end{equation}
Both \(e_0\) and \(e_{\mathrm{old}}\) are defined over the same deterministic
candidate set \(\mathcal C(q)\) included in the prompt. They are obtained by
softmax-normalizing length-normalized LLM token scores over candidate items:
\[
e_p(a\mid q)
=
\frac{\exp(s_p(a,q)/\tau)}
{\sum_{b\in\mathcal C(q)}\exp(s_p(b,q)/\tau)},
\qquad
p\in\{0,\mathrm{old}\}.
\]
Thus, Eq.~\eqref{eq:anchor_ips_ratio} is an item-level propensity ratio,
not a ratio of raw token-level likelihoods or sequence probabilities.

Under the absolute-continuity condition
\begin{equation}
e_{\mathrm{old}}(a\mid q)>0
\quad\Rightarrow\quad
e_0(a\mid q)>0,
\label{eq:absolute_continuity}
\end{equation}
we have the anchor-level IPS identity
\begin{equation}
\mathbb{E}_{a\sim e_0(\cdot\mid q)}
\left[
w(a,q)f_y(a,q)
\right]
=
\mathbb{E}_{a\sim e_{\mathrm{old}}(\cdot\mid q)}
\left[
f_y(a,q)
\right].
\label{eq:anchor_ips_identity}
\end{equation}
This identity only concerns the scalar logged-anchor contribution under
item-level propensities. It does not imply that the complete ABPO objective
is an unbiased off-policy estimator.

\paragraph{Practical self-normalized weighting.}
In practice, for each logged example \(i\), we use the stopped-gradient
item-level propensity ratio
\begin{equation}
w_i^{\log}
=
\operatorname{sg}
\left[
\frac{
e_{\mathrm{old}}(a_i^{\log}\mid q_i)
}{
e_0(a_i^{\log}\mid q_i)
}
\right].
\label{eq:practical_ips_weight}
\end{equation}
To reduce scale imbalance across feedback types, we self-normalize the
weights separately for positive-response and no-response anchors.
Let \(\mathcal B_1\) and \(\mathcal B_0\) denote the corresponding subsets in
a mini-batch. For \(i\in\mathcal B_y\), define
\begin{equation}
\bar w_y^{\log}
=
\frac{1}{|\mathcal B_y|}
\sum_{k\in\mathcal B_y} w_k^{\log},
\qquad
\hat w_i^{\log}
=
\frac{
w_i^{\log}
}{
\bar w_y^{\log}+\delta
}.
\label{eq:practical_snips_weight}
\end{equation}
The feedback-stratified normalization is used as a practical scale-control
device across positive- and no-response anchors, rather than as a claim of
unbiased counterfactual estimation within each feedback-conditioned stratum.

The resulting weight is used only in the anchored baseline:
\begin{equation}
\bar r^{\,w}_{i,\mathrm{anc}}
=
\frac{
\hat w_i^{\log} r_i^{\log}
+
\sum_{j=2}^{G}r_i^j
}{
\hat w_i^{\log}+G-1
}.
\label{eq:practical_weighted_baseline}
\end{equation}
The corresponding weighted standard deviation is
\begin{equation}
\sigma_i
=
\sqrt{
\frac{
\hat w_i^{\log}
\left(r_i^{\log}-\bar r^{\,w}_{i,\mathrm{anc}}\right)^2
+
\sum_{j=2}^{G}
\left(r_i^j-\bar r^{\,w}_{i,\mathrm{anc}}\right)^2
}{
\hat w_i^{\log}+G-1
}
+
\epsilon_{\mathrm{std}}
}.
\label{eq:practical_weighted_std}
\end{equation}
The advantages entering the GRPO surrogate are computed only for
policy-sampled rollouts:
\begin{equation}
A_{i,\mathrm{bandit}}^j
=
\frac{
r_i^j-\bar r^{\,w}_{i,\mathrm{anc}}
}{
\sigma_i
},
\qquad j=2,\ldots,G.
\label{eq:practical_nonanchor_adv}
\end{equation}
There is no term \(A_{i,\mathrm{bandit}}^1\) in the final surrogate, since
the logged anchor is not optimized as a policy-sampled rollout.

\paragraph{Self-normalized anchor-level target.}
For a fixed feedback stratum \(y\), consider logged examples
\(i\in\mathcal B_y\) and a scalar anchor-level quantity
\(f_i=f_y(a_i^{\log},q_i)\). The self-normalized estimator is
\begin{equation}
\hat\mu_{\mathrm{SNIPS}}^{(y)}
=
\frac{
\sum_{i\in\mathcal B_y} w_i^{\log} f_i
}{
\sum_{i\in\mathcal B_y} w_i^{\log}
+
|\mathcal B_y|\delta
}.
\label{eq:snips_estimator}
\end{equation}
Under standard integrability assumptions and i.i.d. sampling within the
feedback stratum, the weak law of large numbers gives
\begin{equation}
\hat\mu_{\mathrm{SNIPS}}^{(y)}
\xrightarrow{p}
\frac{
\mathbb{E}\left[w^{\log} f_y(a^{\log},q)\mid y\right]
}{
\mathbb{E}\left[w^{\log}\mid y\right]+\delta
}.
\label{eq:snips_limit}
\end{equation}
When \(\delta=0\) and exact item-level propensities are available, this
reduces to the usual self-normalized IPS target at the anchor level.

\paragraph{Scope limitation.}
The SNIPS result above should be interpreted only as an anchor-level
exposure-calibration statement based on item-level propensities. Because the
full objective also contains policy-sampled non-anchor rollouts, GRPO
likelihood ratios, feedback-stratified normalization, and reward
normalization within anchor-augmented groups, we do not claim that SNIPS
makes the complete ABPO objective an unbiased estimator of a counterfactual
bandit objective. Instead, SNIPS is used as a practical, stopped-gradient
exposure-correction coefficient for the logged reward's contribution to the
group-relative baseline.


\end{document}